\documentclass[hidelinks,11pt,a4paper]{article}
\usepackage[a4paper, margin=0.95in, footskip=.25in]{geometry}
\usepackage{textcomp}
\usepackage{graphicx}
\usepackage{amsmath,amsfonts,amssymb,amsthm}
\usepackage{tocloft}
\usepackage[font=small,labelfont=bf]{caption}
\usepackage{float}
\usepackage{physics}
\usepackage{ulem} 
\usepackage{tikz}
\usepackage{mathdots}
\usepackage{yhmath}
\usepackage{accents}
\usepackage{cancel}
\usepackage{color}
\usepackage{array}
\usepackage{multirow}
\usepackage{gensymb}
\usepackage{tabularx}
\usepackage{extarrows}
\usepackage{booktabs}
\usetikzlibrary{fadings}
\usetikzlibrary{patterns}
\usetikzlibrary{shadows.blur}
\usetikzlibrary{shapes}
\usepackage{setspace}
\usepackage{hyperref}
\usepackage{sectsty}
\usepackage{url}
\usepackage[nottoc]{tocbibind}
\usepackage{changepage}

\definecolor{subblue}{rgb}{0.29,0.56,0.89}
\definecolor{headers}{RGB}{12, 0, 144}
\newtheorem{theorem}{Theorem}[section]

\newtheorem{lemma}[theorem]{Lemma}
\newcommand{\ent}{\par\vspace{5mm}}

\chapterfont{\color{headers}}  
\sectionfont{\color{headers}}  
\hypersetup{colorlinks=false}
\title{{\LARGE  Evaluating DisCoCirc in Translation Tasks \& its Limitations:\\ A Comparative Study Between Bengali \& English}\\ \vspace{3mm} {\scriptsize This paper is adapted from a mini-project submitted in partial fulfillment of the requirements  for the Master of Science in \\ \vspace{-4mm} Mathematics and Foundations of Computer Science,  University of Oxford, Hilary 2024}\vspace{0mm}}
\author{{\normalsize Nazmoon Falgunee Moon}\\ \small\texttt{nazmoonmoon6@gmail.com}\vspace{0mm}}
\date{\normalsize \today}

\begin{document}

\begin{titlepage}
	\maketitle
	\section*{\color{black} \centering \normalsize Abstract}
	
	\begin{adjustwidth}{30pt}{30pt}
		\small
		In \cite{main}, the authors present the DisCoCirc (Distributed Compositional Circuits) formalism for the English language, a grammar-based framework derived from production rules  that incorporates circuit-like representations in order to give a precise categorical theoretical structure to the language. In this paper, we extend this approach to develop a similar framework for Bengali and apply it to translation tasks between English and Bengali. A central focus of our work lies in reassessing the effectiveness of DisCoCirc in reducing language bureaucracy. Unlike the result suggested in \cite{urdu}, our findings indicate that although it works well for a large part of the language, it still faces limitations due to the structural variation of the two languages. We discuss the possible methods that might handle these shortcomings and show that, in practice, DisCoCirc still struggles even with relatively simple sentences. This divergence from prior claims not only highlights the framework’s constraints in translation but also suggest scope for future improvement. Apart from our primary focus on English–Bengali translation, we also take a short detour to examine English conjunctions, following \cite{and}, showing a connection between conjunctions and Boolean logic.       
\end{adjustwidth}
	 
	\tableofcontents
	\thispagestyle{empty}
\end{titlepage}

\section{Introduction}
The DisCoCirc (Distributed Compositional Circuits) framework  was first introduced in \cite{main} to address the limitations of the previous DisCoCat (Distributional Compositional Categorical) model \cite{Discat}. This framework particularly handles the rigidity and mathematical inconsistencies existent in the earlier model by allowing meanings to evolve across texts and accommodating for a well-defined meaning space by the use of circuit-like diagrams for texts that are both parallely and sequentially composable.
\ent
In the DisCoCirc formalism, texts are represented as circuits, depicting noun phrases as wires and other linguistic units as different processes that modifies them, allowing a precise categorical theocratic structure to the English language.  It is suggested in \cite{main}, that the DisCoCirc formalism might effectively reduce linguistic bureaucracy. Building on this claim, further works are done in \cite{urdu}, where the authors  show that there exists a surjective mapping from the set of  all Urdu text  to the set of all Urdu text circuits. Moreover, for these restricted fragments, the hybrid grammars for English and Urdu are isomorphic, indicating a structural equivalence between the two languages within this framework.
\ent
Our work runs parallel to the studies in \cite{main,urdu}. Here, we start by developing a hybrid grammar for Bengali, construct the corresponding planar tree diagrams, text diagrams, and text circuits,  following the  DisCoCirc formalism of \cite{main}, and examine its structural similarity with English. In addition, we discuss if DisCoCirc can be used as a translation tool between the two languages, revisiting DisCoCirc's capacity in reducing linguistic bureaucracy.
\ent
In addition to our main discussion on translation, we briefly explore English conjunctions, drawing on \cite{and}. While this part of our discussion will not be related to our primary results, it provides a good insight into how DisCoCirc can be related to Boolean logic and how the framework can simplify complex grammatical structures.
\ent
Lastly, we note that throughout this work all diagrams should be read from top to bottom and from left to right. Moreover, for simplicity, we have mostly restricted our discussion to sentences in the simple present tense only.

\section{Hybrid Grammar}
The hybrid grammar system of \cite{main} follows a string rewrite system to generate text sequentially through a finite collection of production rules, where each rule replaces a sequence of symbols with another. Symbols stand either for grammatical units or words. Once a symbol corresponds to a word, no further rewriting applies. Therefore, they are taken as terminal, and we generally underline them. 
\ent
As an example, we can take the production rules for a simple sentence with a transitive verb in English. We can generate valid sentences by step by step replacing the non-terminal symbols in the rules with appropriate phrases or words. That is, if we take:
\begin{align*}
&\texttt{S} \longmapsto \texttt{NP}_1 \cdot \texttt{TVP} \cdot \texttt{NP}_2 
  && \text{\scriptsize (Symbols: $\cdot$= string concatenation; $\texttt{S}$=Start; \texttt{NP} = Noun Phrase; $\texttt{TVP}$ = Transitive Verb Phrase)},\\
&\texttt{NP}_1 \longmapsto \underline{\text{Millie}} 
  && \text{\scriptsize (terminal symbol)},\\
&\texttt{TVP} \longmapsto \underline{\text{eats}} 
  && \text{\scriptsize (terminal symbol)},\\
&\texttt{NP}_2 \longmapsto \underline{\text{rice}} 
  && \text{\scriptsize (terminal symbol)}.
\end{align*}
 we can generate the sentence ``Millie eats rice'' in the following way:
\begin{align*}
    \texttt{S} &\longmapsto \texttt{NP}_1 \cdot \texttt{TVP} \cdot \texttt{NP}_2 \\
    &\longmapsto \text{\underline{Millie}} \cdot \texttt{TVP} \cdot \texttt{NP}_2 \\
    &\longmapsto \text{\underline{Millie}} \cdot \text{\underline{eats}} \cdot \texttt{NP}_2 \\
    &\longmapsto \text{\underline{Millie}} \cdot \text{\underline{eats}} \cdot \text{\underline{rice}}.
\end{align*}

Given a set of production rules, a tree diagram gives a nice visual representation of  the hierarchical structure of a sentence. It starts with the start symbol, $\texttt{S}$, at the top and branches into non-terminal symbols, such as noun phrases (\texttt{NP}), transitive verb phrases (\texttt{TVP}), etc., and further expands to terminal words, like $\underline{\text{Millie}}$, $\underline{\text{rice}}$, etc. For example, using the above production rules, we can generate the following tree diagram:
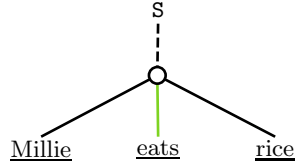
\begin{figure}[H]
	\centering
	\tikzset{every picture/.style={line width=1pt}} 

\begin{tikzpicture}[x=1pt,y=1pt,yscale=-1,xscale=1]

\draw  [dash pattern={on 3.75pt off 1.5pt}]  (367.68,348.17) -- (368,327.53) ;
\draw    (324.22,371.31) -- (367.68,348.17) ;
\draw [color={rgb, 255:red, 126; green, 211; blue, 33 }  ,draw opacity=1 ]   (368,371.31) -- (367.68,348.17) ;
\draw [color={rgb, 255:red, 0; green, 0; blue, 0 }  ,draw opacity=1 ][fill={rgb, 255:red, 65; green, 117; blue, 5 }  ,fill opacity=1 ]   (411.77,371.31) -- (367.68,348.17) ;
\draw  [fill={rgb, 255:red, 255; green, 255; blue, 255 }  ,fill opacity=1 ] (364.56,348.17) .. controls (364.56,346.44) and (365.96,345.04) .. (367.68,345.04) .. controls (369.41,345.04) and (370.81,346.44) .. (370.81,348.17) .. controls (370.81,349.9) and (369.41,351.3) .. (367.68,351.3) .. controls (365.96,351.3) and (364.56,349.9) .. (364.56,348.17) -- cycle ;

\draw (411.77,371.31) node [anchor=north] [inner sep=0.75pt]  [font=\footnotesize] [align=left] {\underline{rice}};
\draw (324.22,371.31) node [anchor=north] [inner sep=0.75pt]  [font=\footnotesize] [align=left] {\underline{Millie}};
\draw (368,371.31) node [anchor=north] [inner sep=0.75pt]  [font=\footnotesize] [align=left] {\underline{eats}};
\draw (368,327.53) node [anchor=south] [inner sep=0.75pt]  [font=\footnotesize] [align=left] {\texttt{S}};

\end{tikzpicture}
	\caption{Tree diagram for a simple sentence (English)}
	\label{fig1}
\end{figure}
Due to structural differences in each language, even direct translated sentences can have different production rules. This results in the generation of different tree diagrams even for sentences with the same meaning, making tree diagrams  great visual indicators of syntactic variation. For example, in Bengali, the production rules for a simple sentence with a transitive verb start with $\texttt{S} \longmapsto \texttt{NP}_1 \cdot \texttt{NP}_2 \cdot \texttt{TVP}$.
The direct translation of ``Millie eats rice.'' is \includegraphics[scale=0.9,trim=0 4.5pt 0 0]{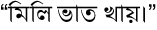} which transliterates to ``Mili $\underset{\color{subblue}rice}{\text{bhāta}}$ $\underset{\color{subblue}eats}{\text{khāẏa.''}}$ So, using the production rules for a simple sentence with a transitive verb in Bengali, we can generate the sentence \includegraphics[scale=0.9,trim=0 4.5pt 0 0]{bengali_assets/Asset_1.eps} as the following: 
\begin{center}
    \includegraphics[scale=0.9]{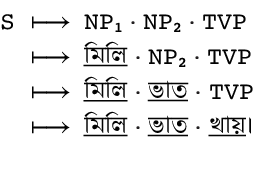}
\end{center}
\vspace{-6mm}
Likewise, the corresponding tree diagram for Bengali will differ from that of English:
\begin{figure}[H]
	\centering
	\includegraphics[scale=0.9]{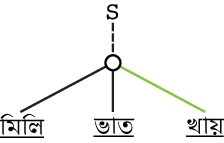}
    \vspace{5mm}
	\caption{Tree diagram for a simple sentence (Bengali)}
	\label{fig1'}
\end{figure}
We now proceed to present the rules and the corresponding tree diagrams for sentences in the Bengali language. From now on, while developing the DisCoCirc formalism for Bengali, we will sometimes include the English structures as well to provide a side by side comparison if needed. Moreover, to improve readability for a wider audience,  we will use transliterated Bengali sentences rather than the original script.

\subsection{Simple Sentence Structure}
 \begin{table}[H]
	\centering
	\include{diagrams/big_table}
	\caption{Rules and planar tree fragments for simple sentence shown side by side for Bengali and English}
	\label{table}
\end{table}
In Tab. \ref{table}, \texttt{NP} denotes a noun phrase, \texttt{TVP} a transitive verb phrase, \texttt{IVP} an intransitive verb phrase, \texttt{ADJ} an adjective phrase, \texttt{ADV} an adverb phrase, and \texttt{ADP} an adposition. We can see that there are  significant differences in the grammatical structures of the two languages, even for simple sentences.
\ent 
\phantomsection \label{is issue}
We begin by examining the second rule for adjectives in Bengali. In English, the 2nd rule for adjectives is for sentences like ``Millie is pretty.'' Here `is (to be)' is  a copula that links the subject of the sentence to the complement, which in our case is an adjective.  Upon translation, this becomes ``Millie $\underset{\color{subblue}{pretty}}{\text{sundarī''}}$ in Bengali, where `is' does not get translated. In Bengali  `is' has a direct translation, which is `haẏa', but in present-tense descriptive sentences like “Millie is pretty,” it is conventional to drop the copula. Translating the copula `is' explicitly  sounds unnatural and often suggests a different meaning. For example, ``Millie $\underset{\color{subblue}{is}}{\text{haẏa}}$ $\underset{\color{subblue}{pretty}}{\text{sundarī''}}$ suggests that it should be interpreted more as ``Millie becomes pretty'' instead. 
\ent
\phantomsection\label{gender issue}
Moreover, Bengali has the concept of  gendered adjectives. They agree with the genders of the nouns they modify. For example, in the sentence ``Millie $\underset{\color{subblue}{pretty}}{\text{sundarī,''}}$ the adjective $\underset{\color{subblue}{pretty}}{\text{`sundarī'}}$ has the suffix `-ī' ($\underset{\color{subblue}{pretty}}{\text{`sundara'}}$ plus `-ī'), which signifies that the noun `Millie' is female. Here, `-ī'  not   a separate syntactic element, rather a part of the adjective’s morphology, a grammatical inflectional ending. It ensures that the adjective aligns with the noun's gender. This is different from English, where adjectives are often invariable and do not reflect the gender of the noun. 
\ent
In English, adverbs can appear both before and after transitive and intransitive verb phrases. For example, in case of a transitive verb, we find sentences like ``Millie rarely sleeps'' and ``Millie runs quickly.'' With an intransitive verb, we see instances like ``Millie often reads the book'' and ``Millie reads the book carefully.'' In contrast, in Bengali, the most natural placement of the adverb is before the verb phrases, whether transitive or intransitive: like ``Millie  $\underset{\color{subblue}{fast}}{\text{druta}}$ $\underset{\color{subblue}{runs}}{\text{ dauṛāẏa''}}$ and ``Millie $\underset{\color{subblue}{often}}{\text{prāẏa'i}}$ chocolate
$\underset{\color{subblue}{eats}}{\text{khāẏa.''}}$ While placing the adverb after the verb is still grammatically correct and conveys the same meaning, it is not often considered natural.\\

Structural differences are also evident in the use of adpositions in the two languages. We will return to this in sec.\ref{case}, pg.\pageref{case}.

\subsection{Compound Sentence Structure}
In Bengali, compound sentences can be formed in several ways: by linking two simple sentences through a pronominal element, by using a verb with a sentential complement, or by connecting clauses with a conjunction. In this section, we present the corresponding production rules and their tree fragments for Bengali. We will see that, though not always, they are somewhat similar to the ones described in \cite{main} for English.
\paragraph{Type I: Pronominal Links:}
 If the same noun occurs as a terminal symbol of two neighboring simple sentences, we can identify them in the respective planar trees of the sentences and replace one of them with an appropriate pronoun, thus merging the sentences. There are four primary ways to achieve this:
 \ent
\textbf{Connection via Personal Pronouns:} In the simplest case, subject-subject links and object-object links between simple sentences can be done by using personal pronouns. In such cases, we identify the occurrence of the same subject or equivalently object in the nearby simple sentence and  replace the noun of the second one with a suitable personal pronoun. For example:
\begin{figure}[H]
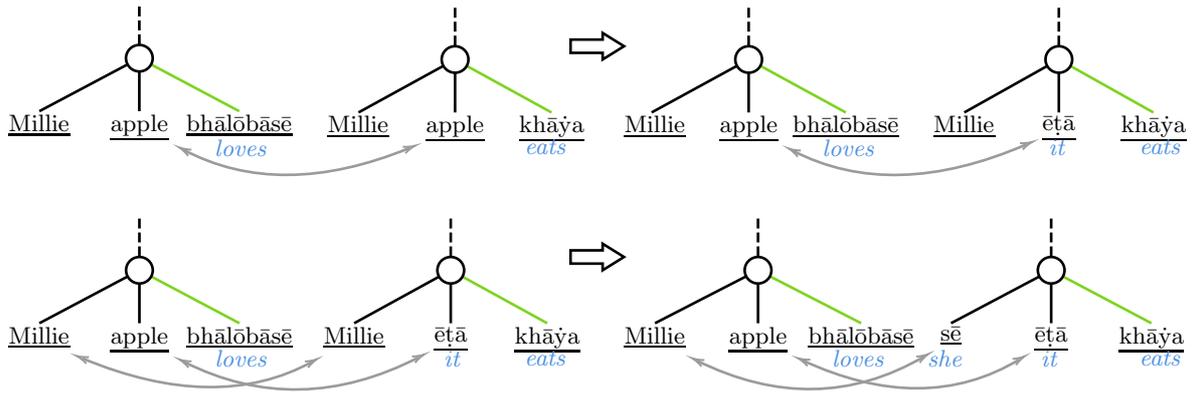

	\centering
	\include{diagrams/polynomiallinks}
	\caption{Examples of pronominal links using personal pronouns}
	\label{fig2}
\end{figure}
\textbf{Connection via Relative Pronouns:}
Connecting two simple sentences using relative pronouns often produces planar tree fragments that are structurally more complex. These types of connection can help link a subject to another subject, an object to a subject, or an object to another object across the sentences.
\ent
To start with, in the  simplest scenario, if we have two simple sentences side by side where the object of the first sentence matches the subject of the second sentence, we can simply replace that subject with $\underset{\color{subblue}{who}}{\text{`yē'}}$ or $\underset{\color{subblue}{which}}{\text{`yā'}}$.  For example:
\begin{figure}[H]
	\centering
	\tikzset{every picture/.style={line width=1pt}} 

\begin{tikzpicture}[x=1pt,y=1pt,yscale=-1,xscale=1]

\draw   (780,422.5) -- (792,422.5) -- (792,420) -- (800,425) -- (792,430) -- (792,427.5) -- (780,427.5) -- cycle ;
\draw  [dash pattern={on 3.75pt off 1.5pt}]  (632,430) -- (632,410) ;
\draw    (592,450) -- (632,430) ;
\draw [color={rgb, 255:red, 0; green, 0; blue, 0 }  ,draw opacity=1 ]   (632,450) -- (632,430) ;
\draw [color={rgb, 255:red, 126; green, 211; blue, 33 }  ,draw opacity=1 ]   (672,450) -- (632,430) ;
\draw  [fill={rgb, 255:red, 255; green, 255; blue, 255 }  ,fill opacity=1 ] (627,430) .. controls (627,427.24) and (629.24,425) .. (632,425) .. controls (634.76,425) and (637,427.24) .. (637,430) .. controls (637,432.76) and (634.76,435) .. (632,435) .. controls (629.24,435) and (627,432.76) .. (627,430) -- cycle ;
\draw  [dash pattern={on 3.75pt off 1.5pt}]  (736.12,430) -- (736.12,410) ;
\draw    (711.95,450) -- (736.12,430) ;
\draw [color={rgb, 255:red, 208; green, 2; blue, 27 }  ,draw opacity=1 ][fill={rgb, 255:red, 208; green, 2; blue, 27 }  ,fill opacity=1 ] [dash pattern={on 3.75pt off 1.5pt}]  (760.3,450) -- (736.12,430) ;
\draw  [fill={rgb, 255:red, 255; green, 255; blue, 255 }  ,fill opacity=1 ] (731.12,430) .. controls (731.12,427.24) and (733.36,425) .. (736.12,425) .. controls (738.89,425) and (741.12,427.24) .. (741.12,430) .. controls (741.12,432.76) and (738.89,435) .. (736.12,435) .. controls (733.36,435) and (731.12,432.76) .. (731.12,430) -- cycle ;
\draw  [dash pattern={on 3.75pt off 1.5pt}]  (850,430) -- (850,410) ;
\draw    (810,450) -- (850,430) ;
\draw [color={rgb, 255:red, 0; green, 0; blue, 0 }  ,draw opacity=1 ]   (850,450) -- (850,430) ;
\draw [color={rgb, 255:red, 126; green, 211; blue, 33 }  ,draw opacity=1 ]   (890,450) -- (850,430) ;
\draw  [fill={rgb, 255:red, 255; green, 255; blue, 255 }  ,fill opacity=1 ] (845,430) .. controls (845,427.24) and (847.24,425) .. (850,425) .. controls (852.76,425) and (855,427.24) .. (855,430) .. controls (855,432.76) and (852.76,435) .. (850,435) .. controls (847.24,435) and (845,432.76) .. (845,430) -- cycle ;
\draw  [dash pattern={on 3.75pt off 1.5pt}]  (951.82,430) -- (951.82,410) ;
\draw    (929.13,450) -- (951.82,430) ;
\draw [color={rgb, 255:red, 208; green, 2; blue, 27 }  ,draw opacity=1 ][fill={rgb, 255:red, 208; green, 2; blue, 27 }  ,fill opacity=1 ] [dash pattern={on 3.75pt off 1.5pt}]  (974.5,450) -- (951.82,430) ;
\draw  [fill={rgb, 255:red, 255; green, 255; blue, 255 }  ,fill opacity=1 ] (946.82,430) .. controls (946.82,427.24) and (949.05,425) .. (951.82,425) .. controls (954.58,425) and (956.82,427.24) .. (956.82,430) .. controls (956.82,432.76) and (954.58,435) .. (951.82,435) .. controls (949.05,435) and (946.82,432.76) .. (946.82,430) -- cycle ;

\draw (672,450) node [anchor=north] [inner sep=0.75pt]  [font=\normalsize] [align=left] {{\footnotesize \underline{bhālōbāsē}}};
\draw (632,450) node [anchor=north] [inner sep=0.75pt]  [font=\footnotesize] [align=left] {\underline{Billiekē}};
\draw (711.95,450) node [anchor=north] [inner sep=0.75pt]  [font=\footnotesize] [align=left] {\underline{Billie}};
\draw (592,450) node [anchor=north] [inner sep=0.75pt]  [font=\footnotesize] [align=left] {\underline{Millie}};
\draw (762,459) node [anchor=north] [inner sep=0.75pt]  [font=\footnotesize,color={rgb, 255:red, 74; green, 144; blue, 226 }  ,opacity=1 ] [align=left] {\textit{handsome}};
\draw (631.5,459) node [anchor=north] [inner sep=0.75pt]  [font=\footnotesize,color={rgb, 255:red, 74; green, 144; blue, 226 }  ,opacity=1 ] [align=left] {\textit{\textcolor[rgb]{0.29,0.56,0.89}{Billie}}};
\draw (760.3,450) node [anchor=north] [inner sep=0.75pt]  [font=\footnotesize] [align=left] {\underline{Sudarśana}};
\draw (672.5,459) node [anchor=north] [inner sep=0.75pt]  [font=\footnotesize,color={rgb, 255:red, 74; green, 144; blue, 226 }  ,opacity=1 ] [align=left] {\textit{\textcolor[rgb]{0.29,0.56,0.89}{loves}}};
\draw (890,450) node [anchor=north] [inner sep=0.75pt]  [font=\normalsize] [align=left] {{\footnotesize \underline{bhālōbāsē}}};
\draw (850,450) node [anchor=north] [inner sep=0.75pt]  [font=\footnotesize] [align=left] {\underline{Billiekē}};
\draw (810,450) node [anchor=north] [inner sep=0.75pt]  [font=\footnotesize] [align=left] {\underline{Millie}};
\draw (975.33,459) node [anchor=north] [inner sep=0.75pt]  [font=\footnotesize,color={rgb, 255:red, 74; green, 144; blue, 226 }  ,opacity=1 ] [align=left] {\textit{handsome}};
\draw (852,459) node [anchor=north] [inner sep=0.75pt]  [font=\footnotesize,color={rgb, 255:red, 74; green, 144; blue, 226 }  ,opacity=1 ] [align=left] {\textit{\textcolor[rgb]{0.29,0.56,0.89}{Billie}}};
\draw (974.5,450) node [anchor=north] [inner sep=0.75pt]  [font=\footnotesize] [align=left] {\underline{Sudarśana}};
\draw (891,459) node [anchor=north] [inner sep=0.75pt]  [font=\footnotesize,color={rgb, 255:red, 74; green, 144; blue, 226 }  ,opacity=1 ] [align=left] {\textit{\textcolor[rgb]{0.29,0.56,0.89}{loves}}};
\draw (928,451) node [anchor=north] [inner sep=0.75pt]  [font=\footnotesize] [align=left] {\underline{yē}};
\draw (928,460) node [anchor=north] [inner sep=0.75pt]  [font=\footnotesize,color={rgb, 255:red, 74; green, 144; blue, 226 }  ,opacity=1 ] [align=left] {\textit{\textcolor[rgb]{0.29,0.56,0.89}{who}}};
\draw [color={rgb, 255:red, 155; green, 155; blue, 155 }  ,draw opacity=1 ]   (649.89,466.16) .. controls (665.41,476.13) and (681.1,475.81) .. (696.98,465.21) ;
\draw [shift={(698.45,464.19)}, rotate = 144.72] [color={rgb, 255:red, 155; green, 155; blue, 155 }  ,draw opacity=1 ][line width=0.75]    (4.37,-1.32) .. controls (2.78,-0.56) and (1.32,-0.12) .. (0,0) .. controls (1.32,0.12) and (2.78,0.56) .. (4.37,1.32)   ;
\draw [shift={(648.15,465)}, rotate = 34.68] [color={rgb, 255:red, 155; green, 155; blue, 155 }  ,draw opacity=1 ][line width=0.75]    (4.37,-1.32) .. controls (2.78,-0.56) and (1.32,-0.12) .. (0,0) .. controls (1.32,0.12) and (2.78,0.56) .. (4.37,1.32)   ;
\draw [color={rgb, 255:red, 155; green, 155; blue, 155 }  ,draw opacity=1 ]   (869.17,466.16) .. controls (886.43,476.11) and (902.73,475.51) .. (918.08,464.37) ;
\draw [shift={(919.5,463.31)}, rotate = 146.95] [color={rgb, 255:red, 155; green, 155; blue, 155 }  ,draw opacity=1 ][line width=0.75]    (4.37,-1.32) .. controls (2.78,-0.56) and (1.32,-0.12) .. (0,0) .. controls (1.32,0.12) and (2.78,0.56) .. (4.37,1.32)   ;
\draw [shift={(867.23,465)}, rotate = 27.5] [color={rgb, 255:red, 155; green, 155; blue, 155 }  ,draw opacity=1 ][line width=0.75]    (4.37,-1.32) .. controls (2.78,-0.56) and (1.32,-0.12) .. (0,0) .. controls (1.32,0.12) and (2.78,0.56) .. (4.37,1.32)   ;

\end{tikzpicture}
	\caption{Example 1 of connection via relative pronouns}
	\label{ke}
\end{figure}
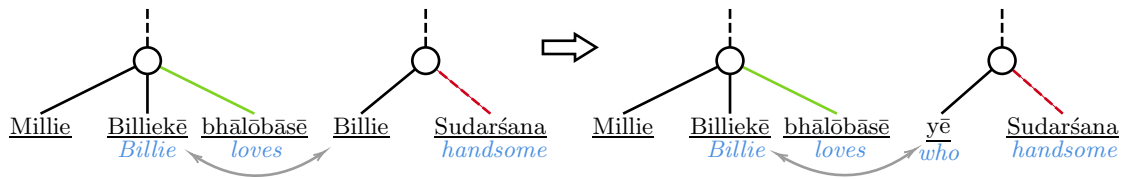
\phantomsection
\label{par:diff}
If both the first and second sentences have the same subject, we can use the pronouns $\underset{\color{subblue}{who}}{\text{`yē}}$\dots $\underset{\color{subblue}{she/he}}{\text{sē'}}$   to connect the sentences. For instance:
\begin{figure}[H]
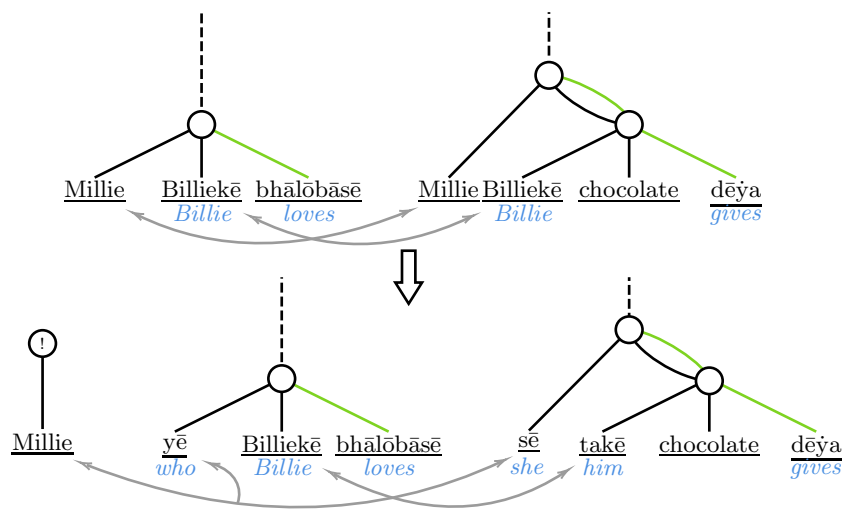

	\centering
	\include{diagrams/plin2}
	\caption{Example 2 of connection via relative pronouns}
	\label{ad}
\end{figure}

Or, if the repeated noun is inanimate, then one can use  $\underset{\color{subblue}{which/that}}{\text{`yā}}$\dots $\underset{\color{subblue}{that/it}}{\text{tā'}}$ as the connecting pronoun. We see that this rule and its subsequent tree diagram are different from the English counterpart \cite{main}, even in the case of direct translation.
In Bengali, if the relative clause  starts with the pronoun 
$\underset{\color{subblue}{who}}{\text{`yē'}}$, 
we almost always need a pronoun in the main clause, like 
$\underset{\color{subblue}{he/she}}{\text{`sē',}}$ 
$\underset{\color{subblue}{he/she (hon.)}}{\text{`tini',}}$
or $\underset{\color{subblue}{they}}{\text{`tārā',}}$ 
to point back to the noun the clause is describing. For example, in ``Millie $\underset{\color{subblue}{who}}{\text{yē}}$ 
Billiekē $\underset{\color{subblue}{loves}}{\text{bhālōbāsē,}}$ 
$\underset{\color{subblue}{she}}{\text{sē}}$ $\underset{\color{subblue}{him}}{\text{tākē}}$ chocolate 
$\underset{\color{subblue}{gives}}{\text{dēẏa,''}}$ 
the part $\underset{\color{subblue}{who}}{\text{`jē}}$ 
Billiekē $\underset{\color{subblue}{loves}}{\text{bhālōbāsē',}}$ 
describes `Millie's preference for `Billie', and the pronoun 
$\underset{\color{subblue}{she}}{\text{`sē'}}$ 
 makes it clear that `Millie' is the one giving the chocolate. If we drop it, as in 
``Millie $\underset{\color{subblue}{who}}{\text{yē}}$ 
Billiekē $\underset{\color{subblue}{loves}}{\text{bhālōbāsē,}}$  
$\underset{\color{subblue}{him}}{\text{tākē}}$ chocolate 
$\underset{\color{subblue}{gives}}{\text{dēẏa,''}}$ the sentence sounds incomplete and confusing because it does not become obvious who is doing the action in the context of Bengali. The only time we can leave out the pronoun $\underset{\color{subblue}{he/she}}{\text{`sē'}}$ is when the main clause does not need a subject. This generally occurs in the case of imperative sentences, for example
$\underset{\color{subblue}{who}}{\text{``Yē}}$ 
$\underset{\color{subblue}{guest}}{\text{atithi}}$ 
$\underset{\color{subblue}{came}}{\text{āschē,}}$ 
$\underset{\color{subblue}{him/her}}{\text{tāke}}$ 
$\underset{\color{subblue}{snacks}}{\text{jalakhābāra}}$ 
$\underset{\color{subblue}{give}}{\text{dāo.''}}$ 
So, while English does not need a resumptive pronoun, in Bengali including $\underset{\color{subblue}{he/she}}{\text{`sē',}}$ 
$\underset{\color{subblue}{he/she (hon.)}}{\text{`tini',}}$ 
or $\underset{\color{subblue}{they}}{\text{`tārā'}}$ is usually necessary.
\ent 
Lastly, though unconventional, if both sentences have identical objects, we can connect them by the relative pronouns $\underset{\color{subblue}{whom}}{\text{`yākē'}}$ or $\underset{\color{subblue}{which}}{\text{`yā'.}}$ This way of connecting two sentences is quite unnatural, but it is still grammatically correct and preserves the meaning. Instead, a more natural approach would be connecting by some suitable conjunctions. An example of such linking is:
\begin{figure}[H]
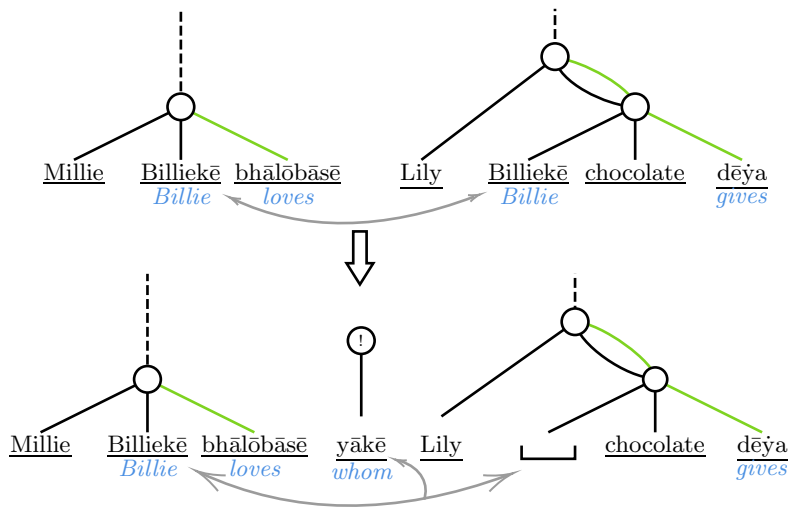

	\centering
	\include{diagrams/plin3}
	\caption{Example 3 of connection via relative pronouns}
	\label{add}
\end{figure}
\textbf{Connection via Reflexive Pronouns:}
Sometimes, in the same sentence,  a single noun can occur twice as both subject and object. In such cases, we use a reflexive pronoun, like:
\begin{figure}[H]
	\centering
	\tikzset{every picture/.style={line width=1pt}} 

\begin{tikzpicture}[x=1pt,y=1pt,yscale=-1,xscale=1]

\draw  [dash pattern={on 3.75pt off 1.5pt}]  (764,840) -- (764,820) ;
\draw    (702,860) -- (764,840) ;
\draw [color={rgb, 255:red, 0; green, 0; blue, 0 }  ,draw opacity=1 ]   (764,860) -- (764,840) ;
\draw [color={rgb, 255:red, 126; green, 211; blue, 33 }  ,draw opacity=1 ]   (830,860) -- (764,840) ;
\draw  [fill={rgb, 255:red, 255; green, 255; blue, 255 }  ,fill opacity=1 ] (759,840) .. controls (759,837.24) and (761.24,835) .. (764,835) .. controls (766.76,835) and (769,837.24) .. (769,840) .. controls (769,842.76) and (766.76,845) .. (764,845) .. controls (761.24,845) and (759,842.76) .. (759,840) -- cycle ;
\draw  [dash pattern={on 3.75pt off 1.5pt}]  (942.22,840) -- (942.22,820) ;
\draw    (880,860) -- (942.22,840) ;
\draw [color={rgb, 255:red, 0; green, 0; blue, 0 }  ,draw opacity=1 ]   (942.22,860) -- (942.22,840) ;
\draw [color={rgb, 255:red, 126; green, 211; blue, 33 }  ,draw opacity=1 ]   (1010,860) -- (942.22,840) ;
\draw  [fill={rgb, 255:red, 255; green, 255; blue, 255 }  ,fill opacity=1 ] (937.22,840) .. controls (937.22,837.24) and (939.45,835) .. (942.22,835) .. controls (944.98,835) and (947.22,837.24) .. (947.22,840) .. controls (947.22,842.76) and (944.98,845) .. (942.22,845) .. controls (939.45,845) and (937.22,842.76) .. (937.22,840) -- cycle ;
\draw   (850,832.5) -- (862,832.5) -- (862,830) -- (870,835) -- (862,840) -- (862,837.5) -- (850,837.5) -- cycle ;

\draw (702,860) node [anchor=north] [inner sep=0.75pt]  [font=\footnotesize] [align=left] {\underline{Millie}};
\draw (766.5,859) node [anchor=north] [inner sep=0.75pt]  [font=\footnotesize] [align=left] {\underline{Milliekē}};
\draw (766.5,868) node [anchor=north] [inner sep=0.75pt]  [font=\footnotesize,color={rgb, 255:red, 74; green, 144; blue, 226 }  ,opacity=1 ] [align=left] {\textit{\textcolor[rgb]{0.29,0.56,0.89}{Millie}}};
\draw (834,859) node [anchor=north] [inner sep=0.75pt]  [font=\footnotesize] [align=left] {\underline{dēkhē}};
\draw (835.5,868) node [anchor=north] [inner sep=0.75pt]  [font=\footnotesize,color={rgb, 255:red, 74; green, 144; blue, 226 }  ,opacity=1 ] [align=left] {\textit{\textcolor[rgb]{0.29,0.56,0.89}{sees}}};
\draw (880,860) node [anchor=north] [inner sep=0.75pt]  [font=\footnotesize] [align=left] {\underline{Millie}};
\draw (945,859) node [anchor=north] [inner sep=0.75pt]  [font=\footnotesize] [align=left] {\underline{nijēkē}};
\draw (945,868) node [anchor=north] [inner sep=0.75pt]  [font=\footnotesize,color={rgb, 255:red, 74; green, 144; blue, 226 }  ,opacity=1 ] [align=left] {\textit{\textcolor[rgb]{0.29,0.56,0.89}{herself}}};
\draw (1014,859) node [anchor=north] [inner sep=0.75pt]  [font=\footnotesize] [align=left] {\underline{dēkhē}};
\draw (1014,868) node [anchor=north] [inner sep=0.75pt]  [font=\footnotesize,color={rgb, 255:red, 74; green, 144; blue, 226 }  ,opacity=1 ] [align=left] {\textit{\textcolor[rgb]{0.29,0.56,0.89}{sees}}};
\draw [color={rgb, 255:red, 155; green, 155; blue, 155 }  ,draw opacity=1 ]   (718.01,871.73) .. controls (728.71,875.82) and (738.15,875.96) .. (746.32,872.15) ;
\draw [shift={(748,871.29)}, rotate = 157.22] [color={rgb, 255:red, 155; green, 155; blue, 155 }  ,draw opacity=1 ][line width=0.75]    (4.37,-1.32) .. controls (2.78,-0.56) and (1.32,-0.12) .. (0,0) .. controls (1.32,0.12) and (2.78,0.56) .. (4.37,1.32)   ;
\draw [shift={(716,870.92)}, rotate = 19.06] [color={rgb, 255:red, 155; green, 155; blue, 155 }  ,draw opacity=1 ][line width=0.75]    (4.37,-1.32) .. controls (2.78,-0.56) and (1.32,-0.12) .. (0,0) .. controls (1.32,0.12) and (2.78,0.56) .. (4.37,1.32)   ;
\draw [color={rgb, 255:red, 155; green, 155; blue, 155 }  ,draw opacity=1 ]   (895.85,872.04) .. controls (906.93,876.59) and (917.56,876.54) .. (927.73,871.9) ;
\draw [shift={(929.5,871.04)}, rotate = 157.61] [color={rgb, 255:red, 155; green, 155; blue, 155 }  ,draw opacity=1 ][line width=0.75]    (4.37,-1.32) .. controls (2.78,-0.56) and (1.32,-0.12) .. (0,0) .. controls (1.32,0.12) and (2.78,0.56) .. (4.37,1.32)   ;
\draw [shift={(894,871.24)}, rotate = 20.12] [color={rgb, 255:red, 155; green, 155; blue, 155 }  ,draw opacity=1 ][line width=0.75]    (4.37,-1.32) .. controls (2.78,-0.56) and (1.32,-0.12) .. (0,0) .. controls (1.32,0.12) and (2.78,0.56) .. (4.37,1.32)   ;

\end{tikzpicture}
	\caption{Example of pronominal link via reflexive pronoun}
	\label{fig9}
\end{figure}
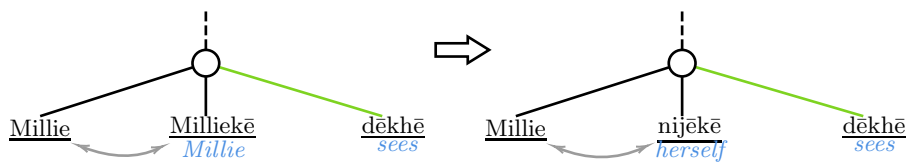
\paragraph{Bengali Accusative Case Marker}\label{case}
At this point in the discussion, it will be useful to look at the Bengali accusative case marker as it is an important part of the hybrid grammar structure of Bengali and something that makes Bengali distinct from English. Instead of giving a full account of the Bengali accusative case marker, we will only look at the cases that came up in our examples.
\ent
\phantomsection \label{ke issue}
To begin with, we see in Fig.\ref{ke}, `Billie' is written as `Billie-kē'. Here  `-kē'   is  an accusative case marker. For this particular case, `-kē'  marks the direct object of the verb. It also indicates `Billie' is an animate entity. If, say, `Billie' were the  name of a book, the sentence would be ``Millie Billie $\underset{\color{subblue}loves}{\text{bhālōbāsē,''}}$ without the `-kē'.
\ent
Another use of accusative case marker in Bengali is to act as an adposition. In  Fig.\ref{ad} and Fig.\ref{add}, we can see that we have used the tree structures for the adposition rule  without using any adposition explicitly (no blue edge). This is because Bengali does not generally use independent prepositions or postpositions in the way English does. Instead, in Bengali, nouns have bound morphemes, like `-kē', `-r', `-tē', `-thēkē', etc., that act similar to adpositions. So, when we see a form like `Billie-kē', we already know `Billie' is the the recipient of the action. In this case, `-kē' roughly corresponds to the English preposition `to', but grammatically it is a case marker rather than an independent adposition.
\ent
\phantomsection
\label{par:prati} 
Another thing we would like to mention here is that Bengali has the concept of a silent or implicit adposition, `prati' that can  act as a postposition. Using it explicitly in sentences does not affect the grammar or the meaning of the sentence, but often it is silent (otherwise, the sentence sounds unnatural), and its meaning is picked up from context. For example, the sentence ``Millie Billier $\underset{\color{subblue}towards}{\text{prati}}$ $\underset{\color{subblue}flower}{\text{phul}}$ $\underset{\color{subblue}throws}{\text{chũṛē''}}$ (``Millie throws flower towards Billie'') is generally expressed simply as  ``Millie Billiekē $\underset{\color{subblue}flower}{\text{phul}}$ $\underset{\color{subblue}throws}{\text{chũṛē,''}}$ where the directional sense of `towards' is understood without the explicit mention of `prati' by the use of accusative case marker `-kē'.  However, if the recipient of action is inanimate it is  more natural to explicitly use `prati' or similar adpositions.  
\paragraph{Type II:}
In Bengali, one can also form a compound sentence either by the use of a verb with sentential complement or suitable conjunctions. We show the production rules and planar tree fragments for such sentences in the following table:
\begin{table}[H]
	\centering
	\include{diagrams/compoundsentence2}
	\caption{Rules and planar tree fragments for compound sentence with conjunction and sentential complement verb}
	\label{table2}
\end{table}

\section{Text Diagrams}
Although hybrid grammar provides a way to describe a language in a methodical way using  rules, links, and boundaries, the formulation has its shortcomings. One obvious problem it faces  is that, as additional structures arise in the language, it starts becoming cumbersome, as it has to introduce new symbols to reason about those structures. For example, in the case of pronominal links, we had to adopt notations such as `\tikz[baseline=(exmark.base)]{
  \node[draw, circle, minimum size=1.0em, line width=0.6pt, inner sep=0pt] (exmark) {\tiny !};
}', `\rotatebox{270}{]}', etc. Text diagrams address this issue by offering a single unified mathematical structure. Unlike tree diagrams, text diagrams are not planar, as they allow lines to intersect. Moreover, the diagrams can be simplified under certain conditions. In the following discussions, we will talk about how we can generate text diagrams from tree planar diagrams for simple and complex sentences in Bengali.

\subsection{Simple Sentence}
 Similar to  \cite{main},  to get text diagrams from tree diagrams in Bengali we will need to replace the sentence types with sentence dependent number of $\texttt{NP}$ wires. This gives us the following  generating text diagrams for simple sentences in Bengali:
\begin{figure}[H]
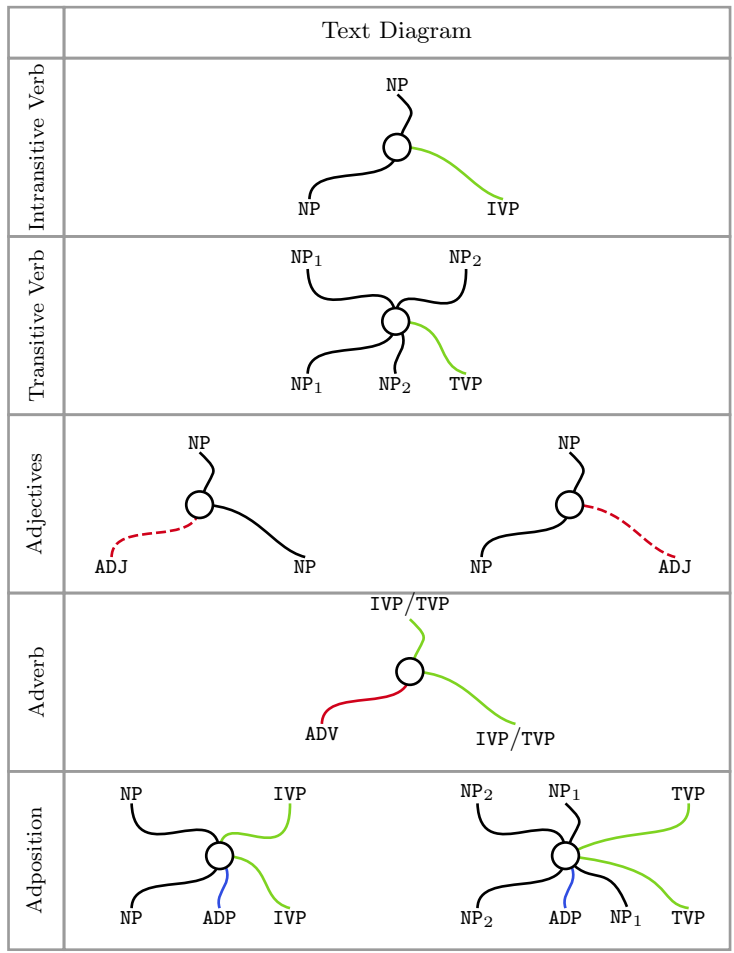

	\centering
	\include{diagrams/bengalitexttree}
	\caption{Text diagrams for Bengali Simple Sentences}
	\label{table3}
\end{figure}

\subsection{Compound Sentence}
\paragraph{Type I: }
For compound sentences, in case of pronominal links, we will sequentially compose the text diagrams in such a way that there is a type match. Moreover, the composition will follow the rules:
\begin{figure}[H]
	\centering
	\tikzset{every picture/.style={line width=1pt}} 

\begin{tikzpicture}[x=1pt,y=1pt,yscale=-1,xscale=1]

\draw [color={rgb, 255:red, 0; green, 0; blue, 0 }  ,draw opacity=1 ]   (170,360) -- (170,330) ;
\draw [color={rgb, 255:red, 0; green, 0; blue, 0 }  ,draw opacity=1 ]   (200,360) -- (200,330) ;
\draw [color={rgb, 255:red, 0; green, 0; blue, 0 }  ,draw opacity=1 ]   (280,360) -- (280,330) ;
\draw [color={rgb, 255:red, 0; green, 0; blue, 0 }  ,draw opacity=1 ]   (310,360) -- (310,330) ;
\draw [color={rgb, 255:red, 0; green, 0; blue, 0 }  ,draw opacity=1 ]   (240,360) -- (240,330) ;
\draw [color={rgb, 255:red, 0; green, 0; blue, 0 }  ,draw opacity=1 ]   (411.25,360) -- (381.25,330) ;
\draw [color={rgb, 255:red, 0; green, 0; blue, 0 }  ,draw opacity=1 ]   (411.25,330) -- (381.25,360) ;
\draw [color={rgb, 255:red, 155; green, 155; blue, 155 }  ,draw opacity=1 ] [dash pattern={on 3.75pt off 2.25pt}]  (381.25,360) .. controls (366.25,360) and (366.25,330) .. (381.25,330) ;
\draw [color={rgb, 255:red, 0; green, 0; blue, 0 }  ,draw opacity=1 ]   (450,360) -- (450,330) ;
\draw [color={rgb, 255:red, 0; green, 0; blue, 0 }  ,draw opacity=1 ]   (523,360) -- (493,330) ;
\draw [color={rgb, 255:red, 0; green, 0; blue, 0 }  ,draw opacity=1 ]   (523,330) -- (493,360) ;
\draw [color={rgb, 255:red, 155; green, 155; blue, 155 }  ,draw opacity=1 ] [dash pattern={on 3.75pt off 2.25pt}]  (523,360) .. controls (538,360) and (538,330) .. (523,330) ;
\draw [color={rgb, 255:red, 155; green, 155; blue, 155 }  ,draw opacity=1 ] [dash pattern={on 3.75pt off 1.5pt}]  (170,330) .. controls (199.5,330.5) and (170.5,360.5) .. (200,360) ;
\draw [color={rgb, 255:red, 155; green, 155; blue, 155 }  ,draw opacity=1 ] [dash pattern={on 3.75pt off 1.5pt}]  (280,360) .. controls (309.5,360.5) and (280.5,330.5) .. (310,330) ;

\draw (339,345) node   [align=left] {and};
\draw (261,346.5) node    {$=$};
\draw (219,346.5) node    {$=$};
\draw (471,346.5) node    {$=$};
\draw (429,346.5) node    {$=$};

\end{tikzpicture}
	\caption{Composition rules for compound sentences with pronominal links}
	\label{fig11}
\end{figure}
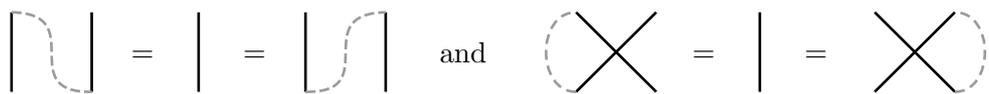
Along with this, the fact that wires are allowed to cross one another during  composition  gives the text diagrams much more fluidity than tree diagrams. This, in turn, allows text diagrams to simplify some of the linguistic complexities inherent in the language of consideration, indicating  more towards structural similarities between languages. For example, in par.\ref{par:diff}, pg.\pageref{par:prati}, we came across a rule for pronominal linking that behaves differently in Bengali and English. But when we shift from tree diagrams to text diagrams, because of the composition rules of the text diagrams, the final complex sentence becomes structurally more similar. For instance, the earlier compound sentence in Fig.\ref{ad} will have the following text diagram:
\begin{figure}[H]
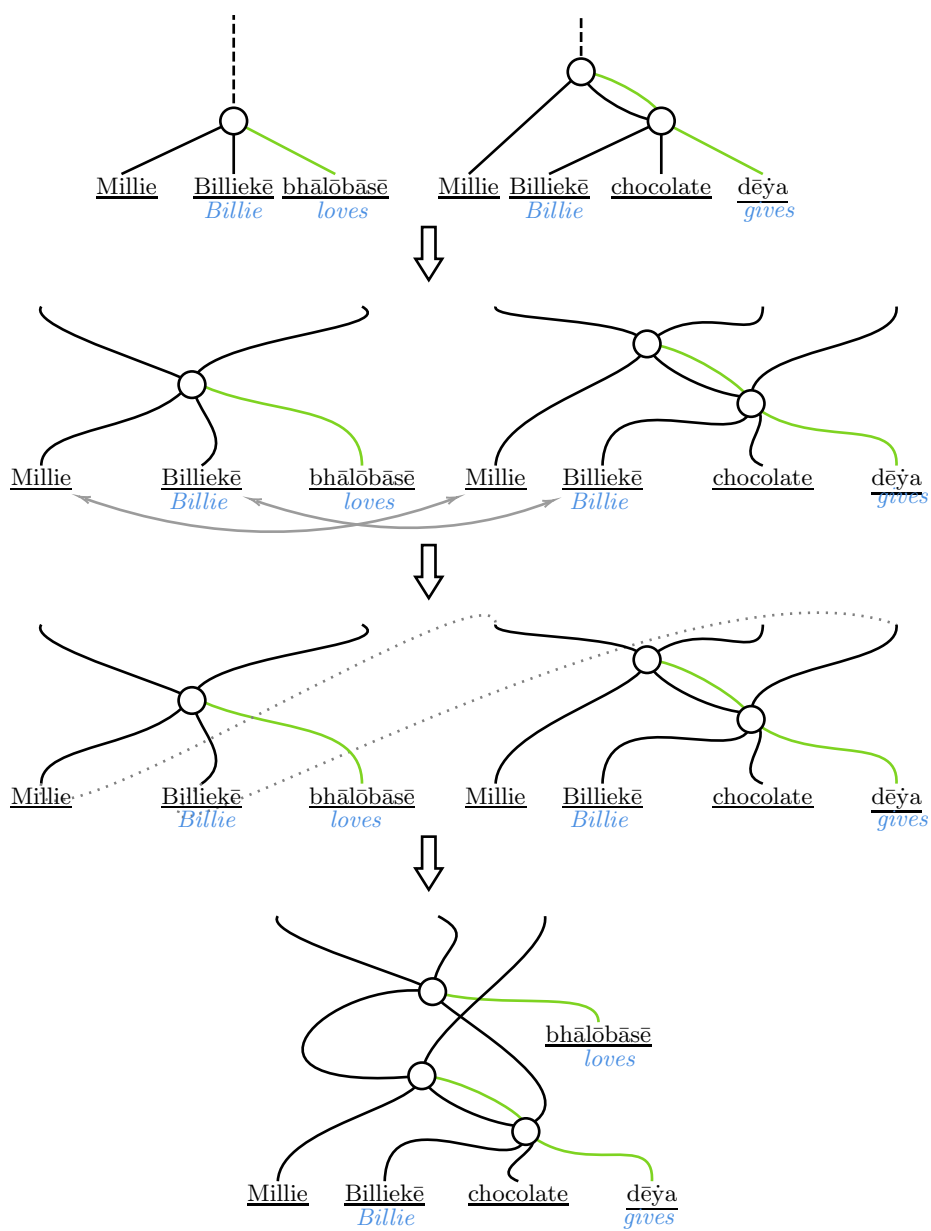

	\centering
	\include{diagrams/TXT_DIAGRAM}
	\caption{Corresponding text diagram of Fig.\ref{ad}}
	\label{table4}
\end{figure}

\paragraph{Type II:}
For compound sentences with sentential complement verbs and conjunctions, we will use phase bubbles  \cite{main}, to get text diagrams for Bengali:
\begin{table}[H]
	\centering
	\tikzset{every picture/.style={line width=1pt}} 

\begin{tikzpicture}[x=1pt,y=1pt,yscale=-1,xscale=1]

\draw  [color={rgb, 255:red, 155; green, 155; blue, 155 }  ,draw opacity=1 ] (50,80) -- (72.06,80) -- (72.06,103.28) -- (50,103.28) -- cycle ;
\draw  [color={rgb, 255:red, 155; green, 155; blue, 155 }  ,draw opacity=1 ] (72.06,80) -- (300,80) -- (300,103.28) -- (72.06,103.28) -- cycle ;
\draw  [color={rgb, 255:red, 155; green, 155; blue, 155 }  ,draw opacity=1 ] (50,103.67) -- (72.06,103.67) -- (72.06,205) -- (50,205) -- cycle ;
\draw  [color={rgb, 255:red, 155; green, 155; blue, 155 }  ,draw opacity=1 ] (72.06,103.28) -- (300,103.28) -- (300,204.62) -- (72.06,204.62) -- cycle ;
\draw  [color={rgb, 255:red, 155; green, 155; blue, 155 }  ,draw opacity=1 ] (50,205) -- (72.06,205) -- (72.06,315.29) -- (50,315.29) -- cycle ;
\draw  [color={rgb, 255:red, 155; green, 155; blue, 155 }  ,draw opacity=1 ] (72.06,205) -- (300,205) -- (300,315.29) -- (72.06,315.29) -- cycle ;
\draw  [fill={rgb, 255:red, 155; green, 155; blue, 155 }  ,fill opacity=1 ] (160.26,146.58) .. controls (160.45,147.11) and (231.95,159.3) .. (186.91,176.46) .. controls (141.88,193.61) and (160.07,146.06) .. (160.26,146.58) -- cycle ;
\draw    (193.08,190) -- (205.33,190) ;
\draw    (130,120) .. controls (125.4,137.16) and (159.65,134.79) .. (160.26,146.58) ;
\draw    (160.26,146.58) .. controls (155.66,163.74) and (129.39,178.21) .. (130,190) ;
\draw [color={rgb, 255:red, 126; green, 211; blue, 33 }  ,draw opacity=1 ]   (160.26,146.58) .. controls (203.76,145.21) and (229.69,173.46) .. (230,190) ;
\draw    (180.79,164.2) .. controls (176.19,181.36) and (169.39,178.21) .. (170,190) ;
\draw    (199.2,190) .. controls (203.8,173.66) and (179.87,179.68) .. (180.79,164.2) ;
\draw    (200,120) .. controls (208.5,139) and (180.17,152.41) .. (180.79,164.2) ;
\draw    (170,120) .. controls (160.5,132) and (179.87,153.17) .. (180.79,164.2) ;
\draw  [fill={rgb, 255:red, 255; green, 255; blue, 255 }  ,fill opacity=1 ] (174.66,158.07) -- (186.91,158.07) -- (186.91,170.33) -- (174.66,170.33) -- cycle ;
\draw  [fill={rgb, 255:red, 155; green, 155; blue, 155 }  ,fill opacity=1 ] (187.73,245.92) .. controls (194.02,238.9) and (296.92,265.18) .. (263.25,283.68) .. controls (229.58,302.19) and (181.43,252.94) .. (187.73,245.92) -- cycle ;
\draw    (256,300) -- (266,300) ;
\draw    (238.08,267.7) .. controls (233.36,285.32) and (229.37,287.89) .. (230,300) ;
\draw    (260,300) .. controls (264.72,279.54) and (237.13,287.08) .. (238.08,267.7) ;
\draw    (260,220) .. controls (268.5,244) and (237.45,255.59) .. (238.08,267.7) ;
\draw    (230,220) .. controls (217.73,233.22) and (237.13,256.37) .. (238.08,267.7) ;
\draw  [fill={rgb, 255:red, 255; green, 255; blue, 255 }  ,fill opacity=1 ] (231.78,261.41) -- (244.37,261.41) -- (244.37,273.99) -- (231.78,273.99) -- cycle ;
\draw  [fill={rgb, 255:red, 155; green, 155; blue, 155 }  ,fill opacity=1 ] (187.17,245.92) .. controls (180.95,238.9) and (79.15,265.18) .. (112.46,283.68) .. controls (145.77,302.19) and (193.4,252.94) .. (187.17,245.92) -- cycle ;
\draw    (156.11,300) -- (149.82,300) -- (143.53,300) ;
\draw    (137.37,267.7) .. controls (142.04,290.31) and (150.44,284.46) .. (149.82,300) ;
\draw    (120,300) .. controls (115.33,284.06) and (138.3,282.8) .. (137.37,267.7) ;
\draw    (120,220) .. controls (107.86,238.25) and (137.99,255.59) .. (137.37,267.7) ;
\draw    (150,220) .. controls (162.14,233.22) and (138.3,256.37) .. (137.37,267.7) ;
\draw  [fill={rgb, 255:red, 255; green, 255; blue, 255 }  ,fill opacity=1 ] (143.59,261.41) -- (131.14,261.41) -- (131.14,273.99) -- (143.59,273.99) -- cycle ;
\draw [color={rgb, 255:red, 144; green, 19; blue, 254 }  ,draw opacity=1 ]   (187.73,245.92) .. controls (188.54,256.94) and (176.47,286.77) .. (190,300) ;
\draw  [fill={rgb, 255:red, 255; green, 255; blue, 255 }  ,fill opacity=1 ] (155.26,146.58) .. controls (155.26,143.82) and (157.5,141.58) .. (160.26,141.58) .. controls (163.02,141.58) and (165.26,143.82) .. (165.26,146.58) .. controls (165.26,149.34) and (163.02,151.58) .. (160.26,151.58) .. controls (157.5,151.58) and (155.26,149.34) .. (155.26,146.58) -- cycle ;
\draw  [fill={rgb, 255:red, 255; green, 255; blue, 255 }  ,fill opacity=1 ] (182.73,245.92) .. controls (182.73,243.16) and (184.97,240.92) .. (187.73,240.92) .. controls (190.49,240.92) and (192.73,243.16) .. (192.73,245.92) .. controls (192.73,248.68) and (190.49,250.92) .. (187.73,250.92) .. controls (184.97,250.92) and (182.73,248.68) .. (182.73,245.92) -- cycle ;

\draw (61.03,154.33) node  [font=\footnotesize,rotate=-270] [align=left] {{\scriptsize Sent. Compt. Verb}};
\draw (61.03,260.15) node  [font=\footnotesize,rotate=-270] [align=left] {{\scriptsize Conjunction}};
\draw (126,300) node [anchor=north] [inner sep=0.75pt]  [font=\scriptsize]  {${{\texttt{NP}}_{\texttt{S}}}_{o1}$};
\draw (236,300) node [anchor=north] [inner sep=0.75pt]  [font=\scriptsize]  {${{\texttt{NP}}_{\texttt{S}}}_{o2}$};
\draw (149.82,300) node [anchor=north] [inner sep=0.75pt]  [font=\scriptsize]  {${~~~~\texttt{\underline{NP}}}_{\texttt{S}_{i1}}$};
\draw (260,300) node [anchor=north] [inner sep=0.75pt]  [font=\scriptsize]  {$~~~~~\underline{\texttt{NP}}_{\texttt{S}_{i2}}$};
\draw (190,300) node [anchor=north] [inner sep=0.75pt]  [font=\scriptsize]  {$\texttt{CNJ}$};
\draw (233.89,218.05) node [anchor=north west][inner sep=0.75pt]    {$\dotsc $};
\draw (125,218) node [anchor=north east][inner sep=0.75pt]  [xscale=-1]  {$\dotsc $};
\draw (130,120) node [anchor=south] [inner sep=0.75pt]  [font=\scriptsize]  {$\texttt{NP}$};
\draw (230,190) node [anchor=north] [inner sep=0.75pt]  [font=\scriptsize]  {$\texttt{SCV}$};
\draw (130,190) node [anchor=north] [inner sep=0.75pt]  [font=\scriptsize]  {$\texttt{NP}$};
\draw (173,190) node [anchor=north] [inner sep=0.75pt]  [font=\scriptsize]  {$\texttt{NP}_{\texttt{S}_{o}}$};
\draw (170,120) node [anchor=south] [inner sep=0.75pt]  [font=\scriptsize]  {$\ \ \ \ \ \ \ \ \ \ \ \ \ \ \texttt{NP}\times (  o+  i)$};
\draw (173.59,117.83) node [anchor=north west][inner sep=0.75pt]    {$\dotsc $};
\draw (175.5,92.5) node  [font=\footnotesize] [align=left] {Text Diagram};
\draw (137.37,267.7) node  [font=\footnotesize]  {$\textit{\texttt{S}}_{1}$};
\draw (238.08,267.7) node  [font=\footnotesize]  {$\textit{\texttt{S}}_{2}$};
\draw (180.79,164.2) node  [font=\footnotesize]  {$\textit{\texttt{S}}$};
\draw (199.2,190) node [anchor=north] [inner sep=0.75pt]  [font=\scriptsize]  {$~~~\underline{\texttt{NP}}_{\texttt{S}_{i}}$};

\end{tikzpicture}
	\caption{Text diagrams for type II Bengali compound sentences }
	\label{table5}
\end{table}
From the above analysis, we  can come to the understanding that text diagrams reduce the structural rigidity of Bengali, making it more comparable to English. Moreover, we see that text diagrams have a structure similar to that of graphs. We can consider the nodes and the terminal symbols of the text diagrams as vertices, and others as $k$-colorable edges, for some $k$, depending on the production rules. Also, for type II compound sentences instead of phase bubbles we can consider parallel composition (disconnected graphs). Furthermore,  given a text, any text diagram can be broken down into some  generating pieces. The generating pieces are text diagrams of simple sentences, which,  if connected in parallel or sequentially by the rules of text diagrams of compound sentences, give us the whole text. Therefore, we can state the following lemma,
\ent
\begin{lemma}\label{lemma:isomorph_bijection}\label{bi 1}
	Let $T_\mathcal{E}$ and $T_\mathcal{B}$ be two texts generated by the production rules of  the English language and  Bengali language, respectively. Let $S_\mathcal{E}$ be the set of all  terminal symbols of $T_\mathcal{E}$ and $S_\mathcal{B}$ be the set of all terminal symbols of $T_\mathcal{B}$. If there is a bijection (in terms of translation) between $S_\mathcal{E}$ and $S_\mathcal{B}$, then the text diagrams $\tau_\mathcal{E}$ of $T_\mathcal{E}$ and $\tau_\mathcal{B}$ of $T_\mathcal{B}$ are isomorphic in the graph theoretic sense. 
\end{lemma}
Here, the bijection, $S_E \approxeq S_B$, is important because otherwise due to the difference in the production rules of  the respective languages at the level of simple sentences (par.\ref{is issue}, pg.\pageref{is issue}), we might run into cases like: 
\begin{figure}[H]
	\centering
	\tikzset{every picture/.style={line width=1pt}} 

\begin{tikzpicture}[x=1pt,y=1pt,yscale=-1,xscale=1]

\draw    (510,140) .. controls (524,148.2) and (497,152.2) .. (510,160) ;
\draw [color={rgb, 255:red, 126; green, 211; blue, 33 }  ,draw opacity=1 ]   (510,160) .. controls (524,168.2) and (497,172.2) .. (510,180) ;
\draw    (510,160) .. controls (524,168.2) and (447,172.2) .. (460,180) ;
\draw [color={rgb, 255:red, 208; green, 2; blue, 27 }  ,draw opacity=1 ] [dash pattern={on 4.5pt off 4.5pt}]  (560,180) .. controls (562,170.2) and (528,170.2) .. (510,160) ;
\draw    (661.5,140) .. controls (675.5,148.2) and (648.5,152.2) .. (661.5,160) ;
\draw    (661.5,160) .. controls (675.5,168.2) and (598.5,172.2) .. (611.5,180) ;
\draw [color={rgb, 255:red, 208; green, 2; blue, 27 }  ,draw opacity=1 ] [dash pattern={on 4.5pt off 4.5pt}]  (711.5,180) .. controls (713.5,170.2) and (679.5,170.2) .. (661.5,160) ;
\draw  [fill={rgb, 255:red, 255; green, 255; blue, 255 }  ,fill opacity=1 ] (505,160) .. controls (505,157.24) and (507.24,155) .. (510,155) .. controls (512.76,155) and (515,157.24) .. (515,160) .. controls (515,162.76) and (512.76,165) .. (510,165) .. controls (507.24,165) and (505,162.76) .. (505,160) -- cycle ;
\draw  [fill={rgb, 255:red, 255; green, 255; blue, 255 }  ,fill opacity=1 ] (656.5,160) .. controls (656.5,157.24) and (658.74,155) .. (661.5,155) .. controls (664.26,155) and (666.5,157.24) .. (666.5,160) .. controls (666.5,162.76) and (664.26,165) .. (661.5,165) .. controls (658.74,165) and (656.5,162.76) .. (656.5,160) -- cycle ;

\draw (560,180) node [anchor=north] [inner sep=0.75pt]  [font=\normalsize] [align=left] {{\footnotesize \underline{handsome}}};
\draw (510,180) node [anchor=north] [inner sep=0.75pt]  [font=\footnotesize] [align=left] {\underline{is}};
\draw (460,180) node [anchor=north] [inner sep=0.75pt]  [font=\footnotesize] [align=left] {\underline{Billie}};
\draw (713.5,188) node [anchor=north] [inner sep=0.75pt]  [font=\footnotesize,color={rgb, 255:red, 74; green, 144; blue, 226 }  ,opacity=1 ] [align=left] {\textit{handsome}};
\draw (713.5,179) node [anchor=north] [inner sep=0.75pt]  [font=\footnotesize] [align=left] {\underline{Sudarśana}};
\draw (489,197) node [anchor=north west][inner sep=0.75pt]  [font=\footnotesize] [align=left] {(English)};
\draw (640,199) node [anchor=north west][inner sep=0.75pt]  [font=\footnotesize] [align=left] {(Bengali)};
\draw (569,149) node [anchor=north west][inner sep=0.75pt]    {$\ncong $};
\draw (611.5,180) node [anchor=north] [inner sep=0.75pt]  [font=\footnotesize] [align=left] {\underline{Billie}};

\end{tikzpicture}
	\caption{Non-isomorphic text diagrams}
	\label{fig12}
\end{figure}
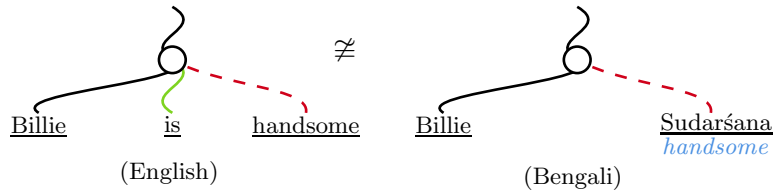

\section{Text Circuits}
The next step after introducing text diagrams will be to move towards text circuits. In the case of English, \cite{main} shows that  text circuits can be represented using wires, boxes, and nested boxes. To begin, we will assume that the text circuits for Bengali have a similar structure to those of English, since our aim is for the text circuits to act as the bridge between the languages. As we progress, we will see, we need to make some slight modifications to the pre-existing formalism to aid our reasoning.
\ent
First, observe, if we take the framework to be the same for Bengali, the issue with the adjective that we discussed in par.\ref{is issue}, pg.\pageref{is issue} gets resolved quite naturally. In the text circuits of the English language \cite{main}, adjectives are treated as processes that modify nouns. As such, the text circuit for English sentences like ``Millie loves Billie. Billie is handsome,'' and their Bengali translation  ``Millie 
$\underset{\color{subblue}{Billie}}{\text{Billiekē}}$ 
$\underset{\color{subblue}{loves}}{\text{bhālōbāsē.}}$
$\underset{\color{subblue}{Billie}}{\text{Billie}}$ 
$\underset{\color{subblue}{handsome}}{\text{sudarśana''}}$ are identical:
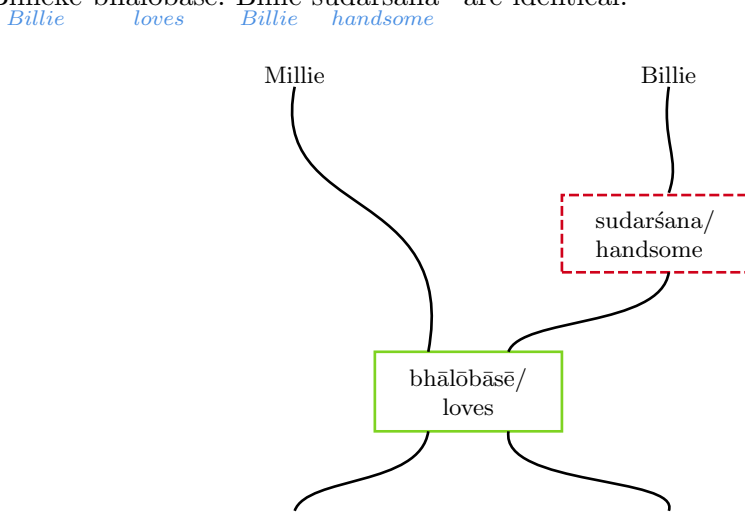
\begin{figure}[H]
	\centering
	\tikzset{every picture/.style={line width=1pt}} 

\begin{tikzpicture}[x=1pt,y=1pt,yscale=-1,xscale=1]

\draw  [color={rgb, 255:red, 126; green, 211; blue, 33 }  ,draw opacity=1 ] (320,770) -- (390,770) -- (390,800) -- (320,800) -- cycle ;
\draw    (290,670) .. controls (279.91,720.45) and (351.64,706.36) .. (340,770) ;
\draw  [color={rgb, 255:red, 208; green, 2; blue, 27 }  ,draw opacity=1 ][dash pattern={on 3.75pt off 1.5pt}] (390,710.86) -- (460,710.86) -- (460,740) -- (390,740) -- cycle ;
\draw    (430,710) .. controls (434.66,696.81) and (426.9,690.96) .. (430,670) ;
\draw    (370,770) .. controls (374.66,756.81) and (426.9,760.96) .. (430,740) ;
\draw    (430,830) .. controls (434.66,816.81) and (366.9,820.96) .. (370,800) ;
\draw    (290,830) .. controls (294.66,816.81) and (336.9,820.96) .. (340,800) ;

\draw (355,785) node  [font=\footnotesize] [align=left] {bhālōbāsē/\\ \ \ \ \ loves};
\draw (430,670) node [anchor=south] [inner sep=0.75pt]  [font=\footnotesize] [align=left] {Billie};
\draw (290,670) node [anchor=south] [inner sep=0.75pt]  [font=\footnotesize] [align=left] {Millie};
\draw (425,725.43) node  [font=\footnotesize] [align=left] {sudarśana/\\handsome};

\end{tikzpicture}
    \caption{Example of identical text circuits despite having different production rules (adjective) }
	\label{tcircuit1}
\end{figure}
In par.\ref{par:prati}, pg \pageref{par:prati}, we discuss  the silent adposition `prati' in Bengali, which corresponds to most  adpositions in English. Although not mentioning `prati' is much more natural in most cases, we will explicitly include it in the text circuits when its existence is clear from the context. For example:
\begin{figure}[H]
	\centering
	\tikzset{every picture/.style={line width=1pt}} 

\begin{tikzpicture}[x=1pt,y=1pt,yscale=-1,xscale=1]

\draw  [color={rgb, 255:red, 126; green, 211; blue, 33 }  ,draw opacity=1 ] (160,400) -- (210,400) -- (210,430) -- (160,430) -- cycle ;
\draw  [color={rgb, 255:red, 48; green, 77; blue, 225 }  ,draw opacity=1 ] (150,390) -- (260,390) -- (260,440) -- (150,440) -- cycle ;
\draw  [color={rgb, 255:red, 126; green, 211; blue, 33 }  ,draw opacity=1 ] (160,330) -- (220,330) -- (220,360) -- (160,360) -- cycle ;
\draw    (110,300) .. controls (110,328.97) and (169.33,303.05) .. (170,330) ;
\draw    (180,300) .. controls (181,328) and (211.35,307.77) .. (210,330) ;
\draw    (200,390) .. controls (190,373) and (282,380) .. (290,300) ;
\draw    (170,360) .. controls (176,377) and (163,377) .. (170,390) ;
\draw    (210,360) .. controls (212,387) and (241,365) .. (240,390) ;
\draw    (170,440) .. controls (170,465) and (115.39,454.5) .. (110,470) ;
\draw    (200,440) .. controls (200,456.84) and (290,451.81) .. (290,470) ;
\draw    (240,440) .. controls (240,456.84) and (180,451.81) .. (180,470) ;

\draw (180,300) node [anchor=south] [inner sep=0.75pt]  [font=\footnotesize] [align=left] {Billie};
\draw (290,300) node [anchor=south] [inner sep=0.75pt]  [font=\footnotesize] [align=left] {chocolate};
\draw (110,300) node [anchor=south] [inner sep=0.75pt]  [font=\footnotesize] [align=left] {Millie};
\draw (185.5,351.5) node  [font=\footnotesize,color={rgb, 255:red, 74; green, 144; blue, 226 }  ,opacity=1 ] [align=left] {\textit{\textcolor[rgb]{0.29,0.56,0.89}{~~~loves}}};
\draw (189.5,339.5) node  [font=\footnotesize] [align=left] {bhālōbāsē};
\draw (239.5,420.5) node  [font=\footnotesize] [align=left] {\textit{\textcolor[rgb]{0.29,0.56,0.89}{to}}};
\draw (236.5,408.5) node  [font=\footnotesize] [align=left] {prati};
\draw (181,421.5) node  [font=\footnotesize,color={rgb, 255:red, 74; green, 144; blue, 226 }  ,opacity=1 ] [align=left] {\textit{~~~gives}};
\draw (185,409.5) node  [font=\footnotesize] [align=left] {dēẏa};

\end{tikzpicture}
    \caption{Associated text circuit of Fig.\ref{ad}}
	\label{tcircuit1_2}
\end{figure}
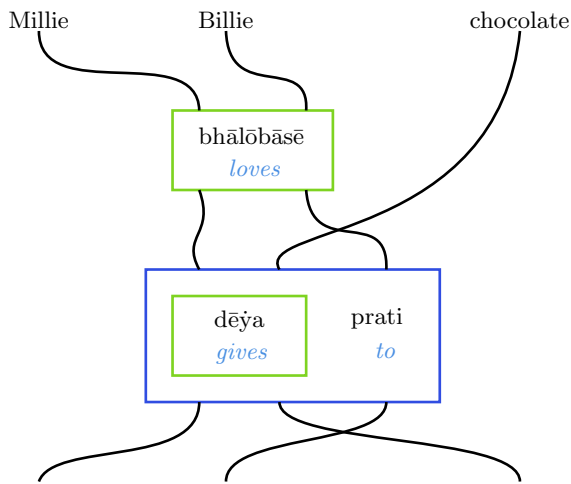
The corresponding text circuit of the direct English translation of the compound sentence mentioned in Fig.\ref{ad} will be:
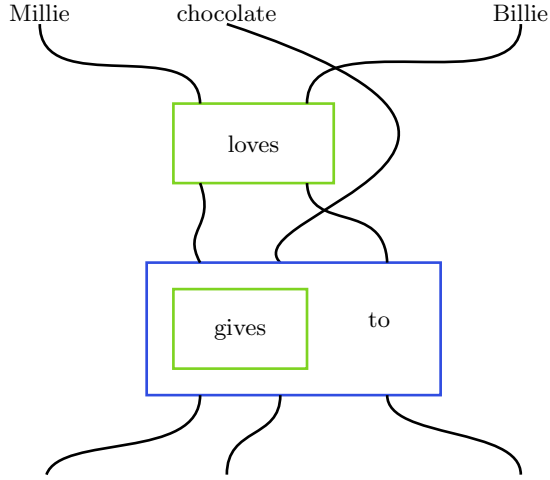
\begin{figure}[H]
	\centering
	\tikzset{every picture/.style={line width=1pt}} 

\begin{tikzpicture}[x=1pt,y=1pt,yscale=-1,xscale=1]

\draw    (520,400) .. controls (502,382) and (642,357) .. (500,310) ;
\draw  [color={rgb, 255:red, 126; green, 211; blue, 33 }  ,draw opacity=1 ] (480,410) -- (530,410) -- (530,440) -- (480,440) -- cycle ;
\draw  [color={rgb, 255:red, 48; green, 77; blue, 225 }  ,draw opacity=1 ] (470,400) -- (580,400) -- (580,450) -- (470,450) -- cycle ;
\draw  [color={rgb, 255:red, 126; green, 211; blue, 33 }  ,draw opacity=1 ] (480,340) -- (540,340) -- (540,370) -- (480,370) -- cycle ;
\draw    (430,310) .. controls (430,338.97) and (489.33,313.05) .. (490,340) ;
\draw    (530,340) .. controls (530,308) and (610.5,340) .. (610,310) ;
\draw    (530,370) .. controls (530.67,388.87) and (559,376) .. (560,400) ;
\draw    (490,450) .. controls (490,475) and (437.89,464.5) .. (432.5,480) ;
\draw    (560,450) .. controls (560,466.84) and (610,461.81) .. (610,480) ;
\draw    (520,450) .. controls (520,466.84) and (500,461.81) .. (500,480) ;
\draw    (490,370) .. controls (496,387) and (483,387) .. (490,400) ;

\draw (610,310) node [anchor=south] [inner sep=0.75pt]  [font=\footnotesize] [align=left] {Billie};
\draw (500,310) node [anchor=south] [inner sep=0.75pt]  [font=\footnotesize] [align=left] {chocolate};
\draw (430,310) node [anchor=south] [inner sep=0.75pt]  [font=\footnotesize] [align=left] {Millie};
\draw (510,355) node  [font=\footnotesize] [align=left] {loves};
\draw (505,425) node  [font=\footnotesize] [align=left] {gives};
\draw (557,421.5) node  [font=\footnotesize] [align=left] {to};

\end{tikzpicture}
    \caption{Text circuit of the direct English translation of the  example in Fig.\ref{ad}}
	\label{tcircuit1_3}
\end{figure}
At first glance, the circuit diagrams of the original and translated texts may appear different. However, since two circuits with identical connectivity but differing only in the ordering of their noun labels are regarded as equivalent \cite{main}, the two diagrams are, in effect, equal up to a translation of the terminal symbols of the text.
\ent
From the above two cases, we see that text circuits can make the underlying structures of the two languages appear more aligned. Therefore, building on Lemma~\ref{bi 1} and noting that connectivity is only significant in text circuits, and that Bengali text circuits adhere to all the principles of text circuits outlined in~\cite{main}, we can state the following lemma:
\ent
\begin{lemma}\label{bi 2}
	Let $T_\mathcal{E}$ be a text in English and $T_\mathcal{B}$ in Bengali, and let $S_\mathcal{E}$ and $S_\mathcal{B}$ be the sets of all terminal symbols of $T_\mathcal{E}$ and $T_\mathcal{B}$, respectively. If there is a bijection (in the sense of translation) between $S_\mathcal{E} $ and $S_\mathcal{B}$,  then the text circuits $C_\mathcal{E}$ of $T_\mathcal{E}$ and $C_\mathcal{B}$ of $T_\mathcal{B}$ are equal up to a translation of the terminal symbols of the text.
\end{lemma} 
In fact, it can be assumed that the text circuits remain equal even under stricter conditions, that is, if there exists a subjective mapping  from $S_\mathcal{E} \setminus \{\underline{\text{is}}\}$ to $S_\mathcal{B} \cup \{\underline{\text{prati}}\}$.

\section{A Detour into Conjunction of English Language}
In \cite{main}, the symbol [\&] is introduced as a generic marker of conjunction to capture the idea of connectedness, which in text circuits  are depicted through  either parallel or sequential composition. In \cite{and}, [\&] is represented as \raisebox{-0.9ex}{
    \begin{tikzpicture}[x=0.75pt,y=0.75pt,yscale=-1,xscale=1]
        \draw    (320,450) .. controls (315,463.2) and (305,463.2) .. (300,450) ;
        \draw [color={rgb, 255:red, 0; green, 0; blue, 0 }  ,draw opacity=1 ]   (310,470) -- (310,460) ;
        \draw  [fill={rgb, 255:red, 255; green, 255; blue, 255 }  ,fill opacity=1 ] (307.5,460) .. controls (307.5,458.62) and (308.62,457.5) .. (310,457.5) .. controls (311.38,457.5) and (312.5,458.62) .. (312.5,460) .. controls (312.5,461.38) and (311.38,462.5) .. (310,462.5) .. controls (308.62,462.5) and (307.5,461.38) .. (307.5,460) -- cycle ;
         
    \end{tikzpicture}
}, which is used as a connector of two noun phrases. Moreover, \raisebox{-0.9ex}{
      \begin{tikzpicture}[x=.75pt,y=.75pt,yscale=-1,xscale=1]

\draw    (280,450) .. controls (275,463.2) and (265,463.2) .. (260,450) ;
\draw [color={rgb, 255:red, 0; green, 0; blue, 0 }  ,draw opacity=1 ]   (270,470) -- (270,460) ;
\draw  [fill={rgb, 255:red, 255; green, 255; blue, 255 }  ,fill opacity=1 ] (267.5,460) .. controls (267.5,458.62) and (268.62,457.5) .. (270,457.5) .. controls (271.38,457.5) and (272.5,458.62) .. (272.5,460) .. controls (272.5,461.38) and (271.38,462.5) .. (270,462.5) .. controls (268.62,462.5) and (267.5,461.38) .. (267.5,460) -- cycle ;
\draw  [dash pattern={on 0.84pt off 2.51pt}]  (280,450) -- (260,450) ;
\end{tikzpicture}} is a generalization of \raisebox{-0.9ex}{
    \begin{tikzpicture}[x=0.75pt,y=0.75pt,yscale=-1,xscale=1]
        \draw    (320,450) .. controls (315,463.2) and (305,463.2) .. (300,450) ;
        \draw [color={rgb, 255:red, 0; green, 0; blue, 0 }  ,draw opacity=1 ]   (310,470) -- (310,460) ;
        \draw  [fill={rgb, 255:red, 255; green, 255; blue, 255 }  ,fill opacity=1 ] (307.5,460) .. controls (307.5,458.62) and (308.62,457.5) .. (310,457.5) .. controls (311.38,457.5) and (312.5,458.62) .. (312.5,460) .. controls (312.5,461.38) and (311.38,462.5) .. (310,462.5) .. controls (308.62,462.5) and (307.5,461.38) .. (307.5,460) -- cycle ;
         
    \end{tikzpicture}
}, which  can connect more than two such units. One can see that  [\&]  acts as the classical AND operation in the text circuit. In addition to [\&], our discussion also requires something similar to the logical NOT operation that will indicate the negation of a word's meaning. Diagrammatically, we will represent it as \raisebox{-0.9ex}{
    \begin{tikzpicture}[x=0.75pt,y=0.75pt,yscale=-1,xscale=1]

        \draw [color={rgb, 255:red, 0; green, 0; blue, 0 }  ,draw opacity=1 ]   (340,470) -- (340,450) ;
        \draw  [fill={rgb, 255:red, 255; green, 255; blue, 255 }  ,fill opacity=1 ] (335,460) .. controls (335,457.24) and (337.24,455) .. (340,455) .. controls (342.76,455) and (345,457.24) .. (345,460) .. controls (345,462.76) and (342.76,465) .. (340,465) .. controls (337.24,465) and (335,462.76) .. (335,460) -- cycle ;
        
        \draw (340,460) node  [font=\tiny] [align=left] {\begin{minipage}[lt]{6.36pt}\setlength\topsep{0pt}
        \begin{center}
        $\displaystyle \pi $
        \end{center}
        
        \end{minipage}};
    \end{tikzpicture}
}.
\begin{lemma}\label{lemma:text_circuit}
In the text circuit of a simple sentence, all conjunctions (except causative ones) between two noun phrases can be written as a composition (parallel or sequential) of \raisebox{-0.9ex}{
    \begin{tikzpicture}[x=0.75pt,y=0.75pt,yscale=-1,xscale=1]
        \draw    (320,450) .. controls (315,463.2) and (305,463.2) .. (300,450) ;
        \draw [color={rgb, 255:red, 0; green, 0; blue, 0 }  ,draw opacity=1 ]   (310,470) -- (310,460) ;
        \draw  [fill={rgb, 255:red, 255; green, 255; blue, 255 }  ,fill opacity=1 ] (307.5,460) .. controls (307.5,458.62) and (308.62,457.5) .. (310,457.5) .. controls (311.38,457.5) and (312.5,458.62) .. (312.5,460) .. controls (312.5,461.38) and (311.38,462.5) .. (310,462.5) .. controls (308.62,462.5) and (307.5,461.38) .. (307.5,460) -- cycle ;
         
    \end{tikzpicture}
} and \raisebox{-0.9ex}{
    \begin{tikzpicture}[x=0.75pt,y=0.75pt,yscale=-1,xscale=1]

        \draw [color={rgb, 255:red, 0; green, 0; blue, 0 }  ,draw opacity=1 ]   (340,470) -- (340,450) ;
        \draw  [fill={rgb, 255:red, 255; green, 255; blue, 255 }  ,fill opacity=1 ] (335,460) .. controls (335,457.24) and (337.24,455) .. (340,455) .. controls (342.76,455) and (345,457.24) .. (345,460) .. controls (345,462.76) and (342.76,465) .. (340,465) .. controls (337.24,465) and (335,462.76) .. (335,460) -- cycle ;
        
        \draw (340,460) node  [font=\tiny] [align=left] {\begin{minipage}[lt]{6.36pt}\setlength\topsep{0pt}
        \begin{center}
        $\displaystyle \pi $
        \end{center}
        
        \end{minipage}};
    \end{tikzpicture}
}. 
\end{lemma}
\begin{proof}
In English, conjunctions, excluding causative ones, primarily have some fixed roles. There are conjunctions  which link  elements together: additive conjunctions;  which depicts options: alternative choice conjunctions;  which  restricts elements or ideas: negation or exclusion conjunctions; and conjunctions representing contrast or opposition of ideas.
\ent
We begin by observing the behavior of these conjunctions. To do so, we consider a simple imperative sentence, ``\underline{\hspace{1cm}} Millie \underline{\hspace{1cm}} Billie go.'' In  \underline{\hspace{1cm}}, we will systematically replace different groups of conjunctions, one by one, so that we can examine how each functions within the sentence.
\ent 
To begin, we focus on conjunctions that signify additive or connectedness. These include `and', `both $\dots$ and', `not only $\dots$ but also', etc. Using our imperative sentence, we can get a sentence like ``Millie and Billie go.''  We can see that these group of conjunctions together behave like the logical AND
operation. In the text circuit, this  is equivalent to:
\begin{figure}[H]
	\centering
	\tikzset{every picture/.style={line width=1pt}} 

\begin{tikzpicture}[x=1pt,y=1pt,yscale=-1,xscale=1]

\draw  [color={rgb, 255:red, 144; green, 19; blue, 254 }  ,draw opacity=1 ] (209.55,162.83) -- (286.36,162.83) -- (286.36,219.15) -- (209.55,219.15) -- cycle ;
\draw  [color={rgb, 255:red, 126; green, 211; blue, 33 }  ,draw opacity=1 ] (214.67,193.55) -- (281.24,193.55) -- (281.24,214.03) -- (214.67,214.03) -- cycle ;
\draw    (219.79,214.03) -- (219.79,222.23) -- (219.79,234.52) ;
\draw    (276.12,214.03) -- (276.12,225.3) -- (276.12,234.52) ;
\draw    (219.79,147.46) .. controls (213.61,169.11) and (239.46,176.77) .. (247.96,175.63) ;
\draw    (276.12,147.46) .. controls (282.53,171.98) and (256.69,176.77) .. (247.96,175.63) ;
\draw    (247.96,175.63) -- (247.96,194.06) ;
\draw  [fill={rgb, 255:red, 255; green, 255; blue, 255 }  ,fill opacity=1 ] (245.4,175.63) .. controls (245.4,177.04) and (246.54,178.19) .. (247.96,178.19) .. controls (249.37,178.19) and (250.52,177.04) .. (250.52,175.63) .. controls (250.52,174.21) and (249.37,173.07) .. (247.96,173.07) .. controls (246.54,173.07) and (245.4,174.21) .. (245.4,175.63) -- cycle ;

\draw (247.96,203.79) node   [align=left] {{\footnotesize go}};
\draw (219.79,147.46) node [anchor=south] [inner sep=0.75pt]   [align=left] {{\footnotesize Millie}};
\draw (276.12,147.46) node [anchor=south] [inner sep=0.75pt]   [align=left] {{\footnotesize Billie}};

\end{tikzpicture}
	\caption{Text circuit for sentence ``Millie and Billie go.'' }
	\label{and}
\end{figure}
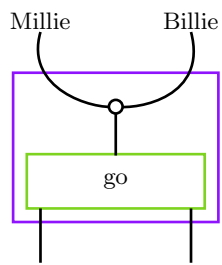
Alternative or choice conjunctions, like `or' and `either\dots or' show options between words or clauses. The options can be both inclusive and exclusive.  For example, in the sentence ``You can have milk or coffee,'' `or' is being used as inclusive or, whereas in ``You can take the car or the train,'' it functions as an indicator of exclusive option. If  used as `inclusive or,' `or' corresponds to the logical OR operation. Then for the sentence ``Millie or Billie goes,'' the text circuit will become:
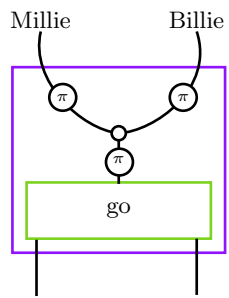
\begin{figure}[H]
	\centering
	\tikzset{every picture/.style={line width=1pt}} 

\begin{tikzpicture}[x=1pt,y=1pt,yscale=-1,xscale=1]

\draw  [color={rgb, 255:red, 144; green, 19; blue, 254 }  ,draw opacity=1 ] (637.21,117.94) -- (716.78,117.94) -- (716.78,187.96) -- (637.21,187.96) -- cycle ;
\draw  [color={rgb, 255:red, 126; green, 211; blue, 33 }  ,draw opacity=1 ] (642.51,161.43) -- (711.47,161.43) -- (711.47,182.65) -- (642.51,182.65) -- cycle ;
\draw    (646.22,183.07) -- (646.22,191.56) -- (646.22,204.29) ;
\draw    (706.17,182.65) -- (706.17,194.32) -- (706.17,203.87) ;
\draw    (647.82,104.48) .. controls (643.49,123.21) and (663.2,141.47) .. (676.99,142.87) ;
\draw    (706.17,104.48) .. controls (713.13,126.07) and (690.25,142.17) .. (676.99,142.87) ;
\draw    (676.99,142.87) -- (676.99,161.96) ;
\draw  [fill={rgb, 255:red, 255; green, 255; blue, 255 }  ,fill opacity=1 ] (674.34,142.87) .. controls (674.34,144.33) and (675.53,145.52) .. (676.99,145.52) .. controls (678.46,145.52) and (679.64,144.33) .. (679.64,142.87) .. controls (679.64,141.4) and (678.46,140.22) .. (676.99,140.22) .. controls (675.53,140.22) and (674.34,141.4) .. (674.34,142.87) -- cycle ;
\draw  [fill={rgb, 255:red, 255; green, 255; blue, 255 }  ,fill opacity=1 ] (696.09,129.34) .. controls (696.09,126.56) and (698.35,124.3) .. (701.13,124.3) .. controls (703.91,124.3) and (706.17,126.56) .. (706.17,129.34) .. controls (706.17,132.12) and (703.91,134.38) .. (701.13,134.38) .. controls (698.35,134.38) and (696.09,132.12) .. (696.09,129.34) -- cycle ;
\draw  [fill={rgb, 255:red, 255; green, 255; blue, 255 }  ,fill opacity=1 ] (651,129.34) .. controls (651,126.56) and (653.25,124.3) .. (656.04,124.3) .. controls (658.82,124.3) and (661.08,126.56) .. (661.08,129.34) .. controls (661.08,132.12) and (658.82,134.38) .. (656.04,134.38) .. controls (653.25,134.38) and (651,132.12) .. (651,129.34) -- cycle ;
\draw  [fill={rgb, 255:red, 255; green, 255; blue, 255 }  ,fill opacity=1 ] (672.22,153.74) .. controls (672.22,150.96) and (674.47,148.7) .. (677.26,148.7) .. controls (680.04,148.7) and (682.3,150.96) .. (682.3,153.74) .. controls (682.3,156.53) and (680.04,158.78) .. (677.26,158.78) .. controls (674.47,158.78) and (672.22,156.53) .. (672.22,153.74) -- cycle ;

\draw (676.99,152.95) node  [font=\tiny]  {$\pi $};
\draw (676.99,172.04) node   [align=left] {{\footnotesize go}};
\draw (647.82,104.48) node [anchor=south] [inner sep=0.75pt]   [align=left] {{\footnotesize Millie}};
\draw (706.17,104.48) node [anchor=south] [inner sep=0.75pt]   [align=left] {{\footnotesize Billie}};
\draw (701.13,129.34) node  [font=\tiny]  {$\pi $};
\draw (656.04,129.34) node  [font=\tiny]  {$\pi $};

\end{tikzpicture}
	\caption{Text circuit for sentence ``Millie or Billie goes.''}
	\label{or}
\end{figure}
The above text diagram can be directly read as ``not (not Millie and not Billie)  go,'' which implies both Millie and Billie can not skip going; that is, at least either of them needs to go, but both can go as well.
\ent
In contrast, in English, conjunctions like `either \dots or' make the choice clearer and show preference for only one option, that is, this conjunction acts like the  XOR operation. Therefore, in  the text circuit format, the sentence ``Either Millie or Billie goes'' can be written as:
\begin{figure}[H]
	\centering
	\tikzset{every picture/.style={line width=1pt}} 

\begin{tikzpicture}[x=1pt,y=1pt,yscale=-1,xscale=1]

\draw  [color={rgb, 255:red, 144; green, 19; blue, 254 }  ,draw opacity=1 ] (730.12,360.84) -- (830.11,360.84) -- (830.11,492.87) -- (730.12,492.87) -- cycle ;
\draw  [color={rgb, 255:red, 126; green, 211; blue, 33 }  ,draw opacity=1 ] (745.59,465.25) -- (817.41,465.25) -- (817.41,487.35) -- (745.59,487.35) -- cycle ;
\draw    (751.11,487.35) -- (751.11,496.19) -- (751.11,509.45) ;
\draw    (781.5,487.9) -- (781.5,500.06) -- (781.5,510) ;
\draw    (761.42,426.36) .. controls (761.42,442.85) and (771.8,451.33) .. (781.11,452.27) ;
\draw    (800.81,426.36) .. controls (801.16,445.21) and (790.06,451.8) .. (781.11,452.27) ;
\draw    (781.11,452.27) -- (781.11,465.16) ;
\draw  [fill={rgb, 255:red, 255; green, 255; blue, 255 }  ,fill opacity=1 ] (779.32,452.27) .. controls (779.32,453.26) and (780.12,454.06) .. (781.11,454.06) .. controls (782.1,454.06) and (782.9,453.26) .. (782.9,452.27) .. controls (782.9,451.28) and (782.1,450.48) .. (781.11,450.48) .. controls (780.12,450.48) and (779.32,451.28) .. (779.32,452.27) -- cycle ;
\draw    (741.17,344.82) .. controls (737.3,376.31) and (752.1,411.57) .. (761.42,413.32) ;
\draw    (780.94,345.37) .. controls (783.71,381.28) and (770.37,412.45) .. (761.42,413.32) ;
\draw    (761.42,413.32) -- (761.42,427.07) ;
\draw    (788.27,376.36) .. controls (788.27,392.85) and (787.2,411.45) .. (800.81,413.6) ;
\draw    (812.97,389.61) .. controls (813.29,408.45) and (808.89,413.13) .. (800.81,413.6) ;
\draw    (800.81,413.6) -- (800.81,426.49) ;
\draw  [fill={rgb, 255:red, 255; green, 255; blue, 255 }  ,fill opacity=1 ] (799.01,413.6) .. controls (799.01,414.59) and (799.82,415.39) .. (800.81,415.39) .. controls (801.79,415.39) and (802.6,414.59) .. (802.6,413.6) .. controls (802.6,412.61) and (801.79,411.81) .. (800.81,411.81) .. controls (799.82,411.81) and (799.01,412.61) .. (799.01,413.6) -- cycle ;
\draw  [fill={rgb, 255:red, 255; green, 255; blue, 255 }  ,fill opacity=1 ] (759.62,413.32) .. controls (759.62,414.31) and (760.43,415.11) .. (761.42,415.11) .. controls (762.4,415.11) and (763.21,414.31) .. (763.21,413.32) .. controls (763.21,412.33) and (762.4,411.53) .. (761.42,411.53) .. controls (760.43,411.53) and (759.62,412.33) .. (759.62,413.32) -- cycle ;
\draw  [fill={rgb, 255:red, 255; green, 255; blue, 255 }  ,fill opacity=1 ] (806.91,402) .. controls (806.91,399.1) and (809.26,396.75) .. (812.16,396.75) .. controls (815.06,396.75) and (817.41,399.1) .. (817.41,402) .. controls (817.41,404.89) and (815.06,407.24) .. (812.16,407.24) .. controls (809.26,407.24) and (806.91,404.89) .. (806.91,402) -- cycle ;
\draw  [fill={rgb, 255:red, 255; green, 255; blue, 255 }  ,fill opacity=1 ] (757.74,435.14) .. controls (757.74,432.24) and (760.09,429.89) .. (762.99,429.89) .. controls (765.89,429.89) and (768.24,432.24) .. (768.24,435.14) .. controls (768.24,438.04) and (765.89,440.39) .. (762.99,440.39) .. controls (760.09,440.39) and (757.74,438.04) .. (757.74,435.14) -- cycle ;
\draw  [fill={rgb, 255:red, 255; green, 255; blue, 255 }  ,fill opacity=1 ] (795.86,434.59) .. controls (795.86,431.69) and (798.21,429.34) .. (801.11,429.34) .. controls (804.01,429.34) and (806.36,431.69) .. (806.36,434.59) .. controls (806.36,437.49) and (804.01,439.84) .. (801.11,439.84) .. controls (798.21,439.84) and (795.86,437.49) .. (795.86,434.59) -- cycle ;
\draw    (742.44,376) -- (788.27,376.36) ;
\draw    (775.97,390.12) -- (812.97,389.61) ;
\draw  [fill={rgb, 255:red, 255; green, 255; blue, 255 }  ,fill opacity=1 ] (784.81,401.44) .. controls (784.81,398.54) and (787.16,396.19) .. (790.06,396.19) .. controls (792.96,396.19) and (795.31,398.54) .. (795.31,401.44) .. controls (795.31,404.34) and (792.96,406.69) .. (790.06,406.69) .. controls (787.16,406.69) and (784.81,404.34) .. (784.81,401.44) -- cycle ;

\draw (790.06,401.44) node  [font=\tiny]  {$\pi $};
\draw (781.5,476.3) node   [align=left] {{\footnotesize go}};
\draw (812.16,402) node  [font=\tiny]  {$\pi $};
\draw (762.99,435.14) node  [font=\tiny]  {$\pi $};
\draw (801.11,434.59) node  [font=\tiny]  {$\pi $};
\draw (741.17,344.82) node [anchor=south] [inner sep=0.75pt]   [align=left] {{\footnotesize Millie}};
\draw (780.94,345.37) node [anchor=south] [inner sep=0.75pt]   [align=left] {{\footnotesize Billie}};

\end{tikzpicture}
	\caption{Text circuit for sentence ``Either Millie or Billie goes.''}
	\label{xor}
\end{figure}
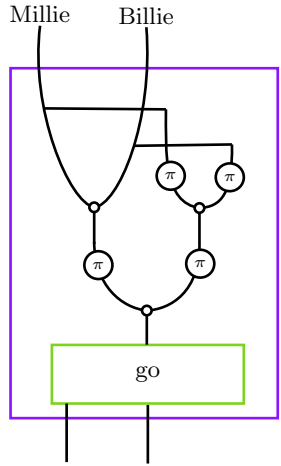
The diagram reads ``Millie and Billie together cannot go and both of them cannot skip going,'' which perfectly captures the meaning of `either\dots or'.
\ent
Conjunctions like `nor', `neither \dots nor', etc., implies negation or total exclusion in a sentence. This equivalents to the Boolean  NOR operation. For example, in text circuit ,``Neither Millie nor Billie goes'' becomes:
\begin{figure}[H]
	\centering
	\tikzset{every picture/.style={line width=1pt}} 

\begin{tikzpicture}[x=1pt,y=1pt,yscale=-1,xscale=1]

\draw  [color={rgb, 255:red, 144; green, 19; blue, 254 }  ,draw opacity=1 ] (810,108.54) -- (885.06,108.54) -- (885.06,174.59) -- (810,174.59) -- cycle ;
\draw  [color={rgb, 255:red, 126; green, 211; blue, 33 }  ,draw opacity=1 ] (815,149.57) -- (880.06,149.57) -- (880.06,169.59) -- (815,169.59) -- cycle ;
\draw    (818.5,169.98) -- (818.5,177.99) -- (818.5,190) ;
\draw    (875.05,169.59) -- (875.05,180.6) -- (875.05,189.61) ;
\draw    (820,95.84) .. controls (814.24,115.66) and (834.52,130.74) .. (847.53,132.06) ;
\draw    (875.05,95.84) .. controls (875.55,122.18) and (860.04,131.4) .. (847.53,132.06) ;
\draw    (847.53,132.06) -- (847.53,150.07) ;
\draw  [fill={rgb, 255:red, 255; green, 255; blue, 255 }  ,fill opacity=1 ] (845.03,132.06) .. controls (845.03,133.44) and (846.15,134.56) .. (847.53,134.56) .. controls (848.91,134.56) and (850.03,133.44) .. (850.03,132.06) .. controls (850.03,130.67) and (848.91,129.55) .. (847.53,129.55) .. controls (846.15,129.55) and (845.03,130.67) .. (845.03,132.06) -- cycle ;
\draw  [fill={rgb, 255:red, 255; green, 255; blue, 255 }  ,fill opacity=1 ] (865.54,119.3) .. controls (865.54,116.67) and (867.67,114.54) .. (870.3,114.54) .. controls (872.92,114.54) and (875.05,116.67) .. (875.05,119.3) .. controls (875.05,121.92) and (872.92,124.05) .. (870.3,124.05) .. controls (867.67,124.05) and (865.54,121.92) .. (865.54,119.3) -- cycle ;
\draw  [fill={rgb, 255:red, 255; green, 255; blue, 255 }  ,fill opacity=1 ] (823.01,119.3) .. controls (823.01,116.67) and (825.14,114.54) .. (827.76,114.54) .. controls (830.39,114.54) and (832.52,116.67) .. (832.52,119.3) .. controls (832.52,121.92) and (830.39,124.05) .. (827.76,124.05) .. controls (825.14,124.05) and (823.01,121.92) .. (823.01,119.3) -- cycle ;

\draw (847.53,159.58) node   [align=left] {{\footnotesize go}};
\draw (820,95.84) node [anchor=south] [inner sep=0.75pt]   [align=left] {{\footnotesize Millie}};
\draw (875.05,95.84) node [anchor=south] [inner sep=0.75pt]   [align=left] {{\footnotesize Billie}};
\draw (870.3,119.3) node  [font=\tiny]  {$\pi $};
\draw (827.76,119.3) node  [font=\tiny]  {$\pi $};

\end{tikzpicture}
	\caption{Text circuit for sentence ``Neither Millie nor Billie goes.''}
	\label{nor}
\end{figure}
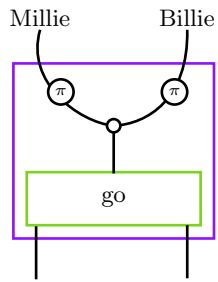
Lastly, contrast or opposition in English is often expressed through conjunctions like `but', `yet', and `however'. Logically, it can be seen as a partial negation. Therefore, text circuit for sentence like ``Millie goes but not Billie'' can be drawn as:
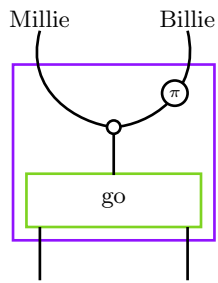
\begin{figure}[H]
	\centering
	\tikzset{every picture/.style={line width=1pt}} 

\begin{tikzpicture}[x=1pt,y=1pt,yscale=-1,xscale=1]

\draw  [color={rgb, 255:red, 144; green, 19; blue, 254 }  ,draw opacity=1 ] (460.8,133.8) -- (536.17,133.8) -- (536.17,200.12) -- (460.8,200.12) -- cycle ;
\draw  [color={rgb, 255:red, 126; green, 211; blue, 33 }  ,draw opacity=1 ] (465.82,175) -- (531.14,175) -- (531.14,195.1) -- (465.82,195.1) -- cycle ;
\draw    (470.85,195.1) -- (470.85,203.14) -- (470.85,215.2) ;
\draw    (526.12,195.1) -- (526.12,206.15) -- (526.12,215.2) ;
\draw    (470.85,121.05) .. controls (464.62,144.16) and (485.42,156.09) .. (498.48,157.41) ;
\draw    (526.12,121.05) .. controls (531.64,145.12) and (511.04,156.75) .. (498.48,157.41) ;
\draw    (498.48,157.41) -- (498.48,175.5) ;
\draw  [fill={rgb, 255:red, 255; green, 255; blue, 255 }  ,fill opacity=1 ] (495.97,157.41) .. controls (495.97,158.8) and (497.09,159.93) .. (498.48,159.93) .. controls (499.87,159.93) and (500.99,158.8) .. (500.99,157.41) .. controls (500.99,156.03) and (499.87,154.9) .. (498.48,154.9) .. controls (497.09,154.9) and (495.97,156.03) .. (495.97,157.41) -- cycle ;
\draw  [fill={rgb, 255:red, 255; green, 255; blue, 255 }  ,fill opacity=1 ] (516.57,144.6) .. controls (516.57,141.96) and (518.71,139.83) .. (521.34,139.83) .. controls (523.98,139.83) and (526.12,141.96) .. (526.12,144.6) .. controls (526.12,147.24) and (523.98,149.37) .. (521.34,149.37) .. controls (518.71,149.37) and (516.57,147.24) .. (516.57,144.6) -- cycle ;

\draw (498.48,185.05) node   [align=left] {{\footnotesize go}};
\draw (470.85,121.05) node [anchor=south] [inner sep=0.75pt]   [align=left] {{\footnotesize Millie}};
\draw (526.12,121.05) node [anchor=south] [inner sep=0.75pt]   [align=left] {{\footnotesize Billie}};
\draw (521.34,144.6) node  [font=\tiny]  {$\pi $};

\end{tikzpicture}
	\caption{Text circuit for sentence ``Millie goes but not Billie.''}
	\label{nor_2}
\end{figure}
Therefore, we see, excluding causal conjunctions, the action of all other conjunctions on two noun phrases can be written as some Boolean operations: Additive  conjunctions by logical AND operation, $A \wedge B$; alternative or choice conjunctions by logical OR operation, $A \vee B$, for inclusive or, and logical XOR operation, $A \oplus B$, for exclusive or; negation or exclusion conjunctions by logical NOR operation, $\neg A \wedge \neg B$; conjunctions showing contrast or opposition via partial negation or $A \wedge \neg B$ operation. Since they all correspond to some Boolean operations, they can be written by some serial and parallel combinations of the AND (\raisebox{-0.9ex}{
    \begin{tikzpicture}[x=0.75pt,y=0.75pt,yscale=-1,xscale=1]
        \draw    (320,450) .. controls (315,463.2) and (305,463.2) .. (300,450) ;
        \draw [color={rgb, 255:red, 0; green, 0; blue, 0 }  ,draw opacity=1 ]   (310,470) -- (310,460) ;
        \draw  [fill={rgb, 255:red, 255; green, 255; blue, 255 }  ,fill opacity=1 ] (307.5,460) .. controls (307.5,458.62) and (308.62,457.5) .. (310,457.5) .. controls (311.38,457.5) and (312.5,458.62) .. (312.5,460) .. controls (312.5,461.38) and (311.38,462.5) .. (310,462.5) .. controls (308.62,462.5) and (307.5,461.38) .. (307.5,460) -- cycle ;
         
    \end{tikzpicture}
}) and the NOT (\raisebox{-0.9ex}{
    \begin{tikzpicture}[x=0.75pt,y=0.75pt,yscale=-1,xscale=1]

        \draw [color={rgb, 255:red, 0; green, 0; blue, 0 }  ,draw opacity=1 ]   (340,470) -- (340,450) ;
        \draw  [fill={rgb, 255:red, 255; green, 255; blue, 255 }  ,fill opacity=1 ] (335,460) .. controls (335,457.24) and (337.24,455) .. (340,455) .. controls (342.76,455) and (345,457.24) .. (345,460) .. controls (345,462.76) and (342.76,465) .. (340,465) .. controls (337.24,465) and (335,462.76) .. (335,460) -- cycle ;
        
        \draw (340,460) node  [font=\tiny] [align=left] {\begin{minipage}[lt]{6.36pt}\setlength\topsep{0pt}
        \begin{center}
        $\displaystyle \pi $
        \end{center}
        
        \end{minipage}};
    \end{tikzpicture}
}) gates.
\end{proof}
\ent
Apart from the  earlier discussed conjunctions, some conjunctions connect clauses by showing cause, result, condition, time, or purpose. Conjunctions like `because', `since', `as' depict cause ($B \to A$);  `so', `therefore', `thus' show result ($A \to B$); `if', `unless' denote conditions ($A \to B$, $\neg A \to B$); `when', `while', `before', `after' indicate time ($A \prec B$, $B \prec A$); `so that', `in order that' express purpose ($A \to B_{\text{goal}}$).  We call all these conjunctions together causal conjunctions. Our lemma does not cover these.
\section{DisCoCirc at Translation}
From Lemma~\ref{bi 1} and Lemma~\ref{bi 2}, we can infer that the DisCoCirc formalism might serve as a useful tool for machine translation between languages$-$in our case, between Bengali and English. In this section, we explore this possibility in a greater depth.
\ent
Primarily, we think that to some extent it is possible to use the text circuit formalism for the task of translation, especially if there is a bijection between the set of terminal symbols of the two languages $S_\mathcal{E}$ and $S_\mathcal{B}$. But to do so, we first need to include an update rule in the text circuits, similar to the ones discussed in \cite{update}. This is because, although there are similarities in the structure of both the languages, they do differ to an extent when looked carefully.
\ent 
Firstly, Bengali and English pronoun system exhibits semantic asymmetry and can not be mapped isomorphically. For example, since Bengali does not have the concept of gendered pronouns both the pronouns `he' and `she' of English get translated to the gender-neutral pronoun `sē' of Bengali. Moreover, unlike English, in Bengali, there are pronouns with honorific distinction, `tui'  (intimate), `tumi' (neutral), `apni' (formal / respectful), and when  translated they all get mapped to `you'.  Besides pronouns, Bengali also has gendered adjectives (par.\ref{gender issue}, pg.\pageref{gender issue}) and differential object markers (par.\ref{ke issue}, pg.\pageref{ke issue}), neither of which is found in English. In addition to this, the subject–verb agreement works differently for the two languages. In English, the form of a verb changes depending on whether the subject is singular or plural, and also on person: for example, ``He/She eats,'' and ``They eat.'' Whereas, in Bengali, a verb must agree with the honorific level of the subject, and similarly on person: for instance, $\underset{\color{subblue}{He/She}}{\text{``Sē}}$ $\underset{\color{subblue}{eats}}{\text{khay,''}}$ and
$\underset{\color{subblue}{They}}{\text{``Tārā}}$
$\underset{\color{subblue}{eat}}{\text{khay''}}$ have the same form of `eat', `khay'; but for the sentence
$\underset{\color{subblue}{He/She (hon.)}}{\text{``Tini}}$
$\underset{\color{subblue}{Billie}}{\text{khan,''}}$ `eat' is written as `khan'.
\ent
To tackle the above mentioned cases, the update rule  discussed in \cite{update} seems helpful. The update rule updates the attributes associated to the noun phrase depending on the context. In the text circuit, this can be done by adding an additional node to the wire of the corresponding noun phrase of our need, like \raisebox{-0.9ex}{
\begin{tikzpicture}[x=0.75pt,y=0.75pt,yscale=-1,xscale=1]

\draw    (440.5,451.5) .. controls (442.21,457.28) and (417.2,454.29) .. (409.52,454.29) ;
\draw [color={rgb, 255:red, 0; green, 0; blue, 0 }  ,draw opacity=1 ]   (410.5,466.5) -- (410.5,436.5) ;
\draw  [fill={rgb, 255:red, 255; green, 255; blue, 255 }  ,fill opacity=1 ] (408.48,453.52) .. controls (408.48,452.41) and (409.38,451.5) .. (410.5,451.5) .. controls (411.62,451.5) and (412.52,452.41) .. (412.52,453.52) .. controls (412.52,454.64) and (411.62,455.55) .. (410.5,455.55) .. controls (409.38,455.55) and (408.48,454.64) .. (408.48,453.52) -- cycle ;
\draw   (441.34,440) -- (465.5,451.5) -- (415.58,451.5) -- cycle ;

\draw (440.5,452.5) node [anchor=south] [inner sep=0.75pt]  [font=\tiny] [align=left] {update};

\end{tikzpicture}

}. As the text progresses, more and more traits get added to the noun as the context becomes clearer. For example, from the English sentences ``Millie loves Billie,'' we have no knowledge of whether `Billie' is a person or an object: therefore, we do not know whether to add `-kē' or not during Bengali translation. But if the text is ``Millie loves Billie. He is handsome,'' we definitely know that `Billie' is a person. This way, as the text progresses, more information is introduced, which adds context to the noun phrases and aids in a more accurate translation.  Thus, the appropriate text diagram with update rule  for this particular text will be:
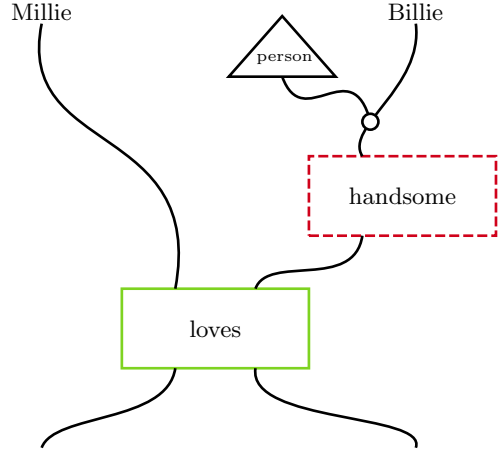
\begin{figure}[H]
	\centering
	\tikzset{every picture/.style={line width=1pt}} 

\begin{tikzpicture}[x=1pt,y=1pt,yscale=-1,xscale=1]

\draw  [color={rgb, 255:red, 126; green, 211; blue, 33 }  ,draw opacity=1 ] (320,770) -- (390,770) -- (390,800) -- (320,800) -- cycle ;
\draw    (290,670) .. controls (279.91,720.45) and (351.64,706.36) .. (340,770) ;
\draw  [color={rgb, 255:red, 208; green, 2; blue, 27 }  ,draw opacity=1 ][dash pattern={on 3.75pt off 1.5pt}] (390,720) -- (460,720) -- (460,750) -- (390,750) -- cycle ;
\draw    (410,720) .. controls (403,709) and (433,694) .. (430,670) ;
\draw    (370,770) .. controls (374.66,756.81) and (406.9,770.96) .. (410,750) ;
\draw    (430,830) .. controls (434.66,816.81) and (366.9,820.96) .. (370,800) ;
\draw    (290,830) .. controls (294.66,816.81) and (336.9,820.96) .. (340,800) ;
\draw    (380,690) .. controls (388,712) and (406,680) .. (413,707) ;
\draw  [fill={rgb, 255:red, 255; green, 255; blue, 255 }  ,fill opacity=1 ] (416,707) .. controls (416,705.34) and (414.66,704) .. (413,704) .. controls (411.34,704) and (410,705.34) .. (410,707) .. controls (410,708.66) and (411.34,710) .. (413,710) .. controls (414.66,710) and (416,708.66) .. (416,707) -- cycle ;
\draw   (380,667.14) -- (400,690) -- (360,690) -- cycle ;

\draw (355,785) node  [font=\footnotesize] [align=left] {loves};
\draw (430,670) node [anchor=south] [inner sep=0.75pt]  [font=\footnotesize] [align=left] {Billie};
\draw (290,670) node [anchor=south] [inner sep=0.75pt]  [font=\footnotesize] [align=left] {Millie};
\draw (425,735) node  [font=\footnotesize] [align=left] {handsome};
\draw (380.86,683.43) node  [font=\tiny] [align=left] {person};

\end{tikzpicture}
	\caption{Example of updated text circuit (English) }
	\label{and1}
\end{figure}
We will also require to use similar update rule in Bengali. From ``Millie 
$\underset{\color{subblue}{Billie}}{\text{Billiekē}}$ 
$\underset{\color{subblue}{loves}}{\text{bhālōbāsē.}}$
$\underset{\color{subblue}{He/She}}{\text{Sē}}$ 
$\underset{\color{subblue}{handsome}}{\text{sudarśana,''}}$ it can be deduced that `Billie' is male, therefore, `sē' in this case should be translated to `he', not `she'. So, in this case, the updated text circuit for Bengali is:
\begin{figure}[H]
	\centering
	\tikzset{every picture/.style={line width=1pt}} 

\begin{tikzpicture}[x=1pt,y=1pt,yscale=-1,xscale=1]

\draw  [color={rgb, 255:red, 126; green, 211; blue, 33 }  ,draw opacity=1 ] (320,770) -- (390,770) -- (390,800) -- (320,800) -- cycle ;
\draw    (290,670) .. controls (279.91,720.45) and (351.64,706.36) .. (340,770) ;
\draw  [color={rgb, 255:red, 208; green, 2; blue, 27 }  ,draw opacity=1 ][dash pattern={on 3.75pt off 1.5pt}] (390,720) -- (460,720) -- (460,750) -- (390,750) -- cycle ;
\draw    (410,720) .. controls (403,709) and (433,694) .. (430,670) ;
\draw    (370,770) .. controls (374.66,756.81) and (406.9,770.96) .. (410,750) ;
\draw    (430,830) .. controls (434.66,816.81) and (366.9,820.96) .. (370,800) ;
\draw    (290,830) .. controls (294.66,816.81) and (336.9,820.96) .. (340,800) ;
\draw    (380,690) .. controls (388,712) and (406,680) .. (413,707) ;
\draw  [fill={rgb, 255:red, 255; green, 255; blue, 255 }  ,fill opacity=1 ] (416,707) .. controls (416,705.34) and (414.66,704) .. (413,704) .. controls (411.34,704) and (410,705.34) .. (410,707) .. controls (410,708.66) and (411.34,710) .. (413,710) .. controls (414.66,710) and (416,708.66) .. (416,707) -- cycle ;
\draw   (380,667.14) -- (400,690) -- (360,690) -- cycle ;

\draw (430,670) node [anchor=south] [inner sep=0.75pt]  [font=\footnotesize] [align=left] {Billie};
\draw (290,670) node [anchor=south] [inner sep=0.75pt]  [font=\footnotesize] [align=left] {Millie};
\draw (355,785) node  [font=\footnotesize] [align=left] {bhālōbāsē};
\draw (425,735) node  [font=\footnotesize] [align=left] {sudarśana};
\draw (380,685.58) node  [font=\tiny] [align=left] {\textit{\textcolor[rgb]{0.29,0.56,0.89}{male}}};
\draw (380.5,680.5) node  [font=\tiny] [align=left] {puruṣa};

\end{tikzpicture}
	\caption{Example of updated text circuit (Bengali) }
	\label{and2}
\end{figure}
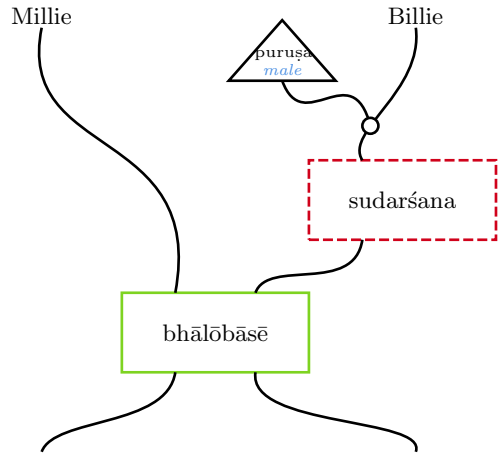
Now, given the text circuit in Fig. \ref{and1}, a possible translation method  might start by converting all the terminal symbols (words) of the text circuit to their corresponding Bengali words. Next, the diagram is to be traced backward along the production rules of Bengali to generate the translated text. However, to obtain a sufficiently accurate translation, there needs to be a bijection (in terms of translation) between the sets of terminal symbols, $S_\mathcal{E}$ and $S_\mathcal{B}$, of the two languages. Otherwise, one of the most obvious issues we will encounter is because of the copula `is' in English (par.\ref{is issue}, pg.\pageref{is issue}). As such, the Bengali text circuit of  Fig. \ref{and2} will not support a direct translation. Even so, following the production rules of English, the sentence can be translated to ``Millie loves handsome Billie,'' which roughly conveys a similar meaning.
\ent
Let us look at another example: ``It drizzles in October.'' Although `October' is  translatable into Bengali, the phrase `In October' is rendered as `Akṭōbarē,'  `October' plus the accusative case marker `-ē.' Unlike the case where the accusative case marker corresponding to `to', `towards' could be translated to the silent apposition `prati' (par.\ref{par:prati}, pg.\pageref{par:prati}), in this case, there is no silent adposition that could do the same keeping meaning intact. Moreover, the first part `It drizzles' is not at all directly translatable. However, in Bengali there is a sentence with similar meaning, ``$\underset{\color{subblue}grain-like}{\text{Gũṛi gũṛi}}$ $\underset{\color{subblue}rain}{\text{brṣṭi}}$ $\underset{\color{subblue}falls}{\text{pore,''}}$ but it is not the exact word to word translation. This is not a unique case: there are numerous words, phrases, and sentences in both languages that cannot be translated in an isomorphic fashion. In such cases, the DisCoCirc formalism proves less effective.
\ent
Another challenge  the formalism might face is when dealing with idioms, as idioms are not directly translatable and behave more like terminal symbols. So, the formalism might have difficulties in identifying idiomatic expressions in sentences of both the languages.
\section{Conclusion}
Throughout this paper, we have developed a a formalism for Bengali similar to the one  in \cite{main}, often times showing a side by side comparison with English in order to point out the structural dissimilarities of the two languages. Our primary goal was to reanalyze how effectively the DisCoCirc formalism can minimize interlingual bureaucracy and if the formalism can aid in machine translation. Through our comparative study of the two languages, we could see that although DisCoCirc works great for a significant portion, it is not adequate enough when confronted with cases where words, phrases or sentences resist direct mapping. Such challenges are especially noticeable when it faces gendered adjectives, adpositions, accusative case markers, pronouns, etc., as these lack direct correspondences between the two languages. Besides, several commonly used  phrases and sentences in both the languages are not translatable at all. Moreover, the formalism faces limitations while dealing with idioms, particularly identifying them. Our present work did not focus much on how translation between Bengali and English might be represented more precisely within the category-theoretic framework. We leave this for future research. Moreover, we plan to investigate how the formalism can be extended to handle causal conjunctions, tense variations, and idioms. By addressing these, we hope to develop DisCoCirc into a more efficient tool capable of dealing with linguistic complexity and translation. 
\section{Acknowledgments}
I would like to thank Rahul Pal for his assistance with typesetting and insights on Bengali language. I would like to acknowledge Bob Coecke and Vincent Wang-Maścianica for their discussions on DisCoCirc. I also thank my parents and Prof. Tibra Ali for their encouragement.

\bibliographystyle{is-abbrv}
\bibliography{ref}

\begin{thebibliography}{1}
\ifx \showCODEN  \undefined \def \showCODEN #1{CODEN #1}  \fi
\ifx \showISBN   \undefined \def \showISBN  #1{ISBN #1}   \fi
\ifx \showISSN   \undefined \def \showISSN  #1{ISSN #1}   \fi
\ifx \showLCCN   \undefined \def \showLCCN  #1{LCCN #1}   \fi
\ifx \showPRICE  \undefined \def \showPRICE #1{#1}        \fi
\ifx \showURL    \undefined \def \showURL {URL }          \fi
\ifx \path       \undefined \input path.sty               \fi
\ifx \ifshowURL \undefined
     \newif \ifshowURL
     \showURLtrue
\fi

\bibitem{and}
B.~Coecke.
\newblock The mathematics of text structure, 2020.
\newblock \ifshowURL {\showURL \path|https://arxiv.org/abs/1904.03478|}\fi.

\bibitem{update}
B.~Coecke and K.~Meichanetzidis.
\newblock Meaning updating of density matrices, 2020.
\newblock \ifshowURL {\showURL \path|https://arxiv.org/abs/2001.00862|}\fi.

\bibitem{Discat}
B.~Coecke, M.~Sadrzadeh, and S.~Clark.
\newblock Mathematical foundations for a compositional distributional model of meaning, 2010.
\newblock \ifshowURL {\showURL \path|https://arxiv.org/abs/1003.4394|}\fi.

\bibitem{main}
V.~Wang-Mascianica, J.~Liu, and B.~Coecke.
\newblock Distilling text into circuits, 2023.
\newblock \ifshowURL {\showURL \path|https://arxiv.org/abs/2301.10595|}\fi.

\bibitem{urdu}
M.~H. Waseem, J.~Liu, V.~Wang-Maścianica, and B.~Coecke.
\newblock Language-independence of discocirc’s text circuits: English and urdu.
\newblock {\em Electronic Proceedings in Theoretical Computer Science}, 366:\penalty0 50–60, Aug. 2022.
\newblock \showISSN{2075-2180}.
\newblock \ifshowURL {\showURL \path|http://dx.doi.org/10.4204/EPTCS.366.7|}\fi.

\end{thebibliography}
\end{document}